\documentclass{article}
\expandafter\def\expandafter\normalsize\expandafter{%
    \normalsize
}
\usepackage{microtype}
\usepackage{graphicx}
\usepackage{subfigure}
\usepackage{booktabs} 
\usepackage{dsfont}
\usepackage{amsthm}
\usepackage{mdwmath}
\usepackage{amssymb}
\usepackage{mdwtab}
\usepackage{amsmath}
\usepackage{makecell}
\usepackage{amsfonts}
\usepackage{enumitem}
\usepackage{caption}
\usepackage{hyperref}

\usepackage[accepted]{icml2023_modified}

\icmltitlerunning{Magnitude Matters: Fixing {\scriptsize SIGN}SGD Through Magnitude-Aware Sparsification in the Presence of Data Heterogeneity}
\newtheorem{theorem}{Theorem}
\newtheorem{Definition}{Definition}

\newtheorem{Corollary}{Corollary}
\newtheorem{Lemma}{Lemma}
\newtheorem{Remark}{Remark}

\newtheorem{Assumption}{Assumption}

\begin{document}
\setlength{\abovedisplayskip}{0pt}
\setlength{\belowdisplayskip}{0pt}
\newcommand\sbullet[1][.5]{\mathbin{\vcenter{\hbox{\scalebox{#1}{$\bullet$}}}}}
\twocolumn[
\icmltitle{Magnitude Matters: Fixing {\large SIGN}SGD Through Magnitude-Aware Sparsification in the Presence of Data Heterogeneity}


\icmlsetsymbol{equal}{*}
\begin{icmlauthorlist}
\icmlauthor{Richeng Jin}{zju}
\icmlauthor{Xiaofan He}{Wu}
\icmlauthor{Caijun Zhong}{zju}
\icmlauthor{Zhaoyang Zhang}{zju}
\icmlauthor{Tony Quek}{sutd}
\icmlauthor{Huaiyu Dai}{ncsu}
\end{icmlauthorlist}

\icmlaffiliation{zju}{Zhejiang University, China.}
\icmlaffiliation{Wu}{Wuhan University, China}
\icmlaffiliation{sutd}{Singapore University of Technology and Design, Singapore.}
\icmlaffiliation{ncsu}{North Carolina State University, USA.}
\icmlcorrespondingauthor{Richeng Jin}{richengjin@zju.edu.cn}

\icmlkeywords{Machine Learning, ICML}

\vskip 0.3in
]


\printAffiliationsAndNotice{}  

\begin{abstract}
Communication overhead has become one of the major bottlenecks in the distributed training of deep neural networks. To alleviate the concern, various gradient compression methods have been proposed, and sign-based algorithms are of surging interest. However, {\scriptsize SIGN}SGD fails to converge in the presence of data heterogeneity, which is commonly observed in the emerging federated learning (FL) paradigm. Error feedback has been proposed to address the non-convergence issue. Nonetheless, it requires the workers to locally keep track of the compression errors, which renders it not suitable for FL since the workers may not participate in the training throughout the learning process. In this paper, we propose a magnitude-driven sparsification scheme, which addresses the non-convergence issue of {\scriptsize SIGN}SGD while further improving communication efficiency. Moreover, the local update scheme is further incorporated to improve the learning performance, and the convergence of the proposed method is established. The effectiveness of the proposed scheme is validated through experiments on Fashion-MNIST, CIFAR-10, and CIFAR-100 datasets.

\end{abstract}

\section{Introduction} \label{Introduction}
\noindent The past decades have witnessed tremendous achievements of deep learning in a variety of areas. In practical training of deep neural networks, it is often difficult, if not impossible, to store the whole training datasets in a single machine. Distributed stochastic gradient descent (D-SGD), which leverages the computational power of multiple computing machines, becomes a promising approach \cite{dean2012large,chen2016revisiting}. D-SGD requires intensive communications between the parameter server and the computing nodes for gradient information exchange. As a result, for the training of large deep neural networks, prohibitive communication overhead becomes one of the major bottlenecks.

There are two popular approaches in the literature to alleviate the communication burden: 1) reducing the number of communication rounds by allowing the computing machines to perform multiple local SGD steps before information exchange (e.g., FedAvg \cite{mcmahan2017communication}, local SGD \cite{stich2018local}); 2) reducing the message size through various compression methods, including quantization \citep{alistarh2017qsgd,wen2017terngrad,wang2018atomo}, sparsification \cite{stich2018sparsified, alistarh2018convergence, lin2018deep,sahu2021rethinking} and their combination \cite{basu2019qsparse}.

Among the quantization schemes, {\scriptsize SIGN}SGD is of particular interest due to its communication efficiency and robustness \cite{bernstein2018signsgd1,bernstein2018signsgd2}. In {\scriptsize SIGN}SGD, the computing machines transmit the signs of the gradients and therefore achieve an improvement of $32\times$ in communication efficiency. While {\scriptsize SIGN}SGD converges with the same rate as D-SGD with homogeneous data \cite{bernstein2018signsgd1}, it fails to converge in the presence of data heterogeneity \cite{chen2019distributed}. Considering that heterogeneous data distribution is one of the major features in the emerging federated learning paradigm, there has been a surging interest in addressing the non-convergence issue of {\scriptsize SIGN}SGD. \citep{chen2019distributed} proposed to add unimodal noise to the gradients before taking the sign and proved a convergence rate of $O(d^{3/4}/T^{1/4})$, in which $d$ is the dimension of the gradients and $T$ the total number of communication rounds. \cite{karimireddy2019error,zheng2019communication} addressed the non-convergence issue of {\scriptsize SIGN}SGD by incorporating the error-feedback mechanism, while \cite{safaryan2021stochastic} proposed stochastic sign gradient descent with momentum (SSDM). Nonetheless, both the error-feedback and the momentum mechanisms require the workers to store errors or momentum locally, and therefore may not be directly applicable in FL in which only a portion of the workers participate in the training process during each round.

In this work, we further extend the theory on sign-based gradient descent methods to the ternary case. Instead of the signs, each coordinate of the gradients is mapped to the ternary $\{-1,0,+1\}$. It essentially captures two important mechanisms to further improve communication efficiency: 1) sparsification, and 2) worker sampling. More specifically, each coordinate of the gradient is mapped to $0$ due to either sparsification or the fact that the worker is not sampled{\color{black}/selected} for training. To this end, we first derive a sufficient condition for the convergence of the ternary-based gradient descent method, based on which a magnitude-driven sparsification mechanism is proposed to address the non-convergence issue of {\scriptsize SIGN}SGD. Then, the local update scheme on the worker side and the error-feedback scheme on the server side are incorporated to further improve the learning performance. The contribution of this paper is summarized as follows:

\begin{enumerate}[topsep=0pt,parsep=0pt,partopsep=0pt]
    \item We extend the convergence analysis for {\color{black}sign-based gradient descent methods} to ternary-based gradient descent in the presence of data heterogeneity, which takes two commonly adopted mechanisms in FL into consideration: sparsification and worker sampling.
    \item Based on the analyses, a magnitude-driven sparsification mechanism is proposed to further improve the communication efficiency of {\scriptsize SIGN}SGD while addressing its non-convergence issue.
    \item We further incorporate the local update scheme at the worker side and the error-feedback mechanism at the parameter server side to improve the learning performance. The convergence of the proposed method is established, and the effectiveness of the proposed method is validated by experiments. 
\end{enumerate}

\section{Related Works}
\noindent 

\textbf{Quantization:} QSGD \citep{alistarh2017qsgd}, TernGrad \citep{wen2017terngrad}, ATOMO \cite{wang2018atomo}, DIANA \cite{mishchenko2019distributed}, FedPAQ \citep{pmlr-v108-reisizadeh20a} and FedCOM \citep{haddadpour2021federated} proposed to use unbiased stochastic quantization schemes. Utilizing the unbiased nature of the adopted quantization methods, the convergence of the corresponding algorithms can be established. \cite{shlezinger2020uveqfed} adopted vector quantization. \cite{horvath2019stochastic} extended DIANA to a variance reduced variant, and \cite{li2020acceleration} proposed an accelerated version of DIANA.

{\color{black}\textbf{Sign-based Methods:} Among all the quantization methods, the sign-based gradient descent method pushes the compression to the limit by representing each coordinate in 1 bit and is therefore of particular research interest.} The idea of sharing the signs of gradients in SGD can be traced back to 1-bit SGD \cite{seide20141}. Although sign-based quantization is biased, \cite{carlson2015stochastic} and \cite{bernstein2018signsgd1,bernstein2018signsgd2} showed theoretical and empirical evidence that sign-based gradient descent schemes achieve comparable performance with the vanilla SGD in the homogeneous data distribution scenario. {\color{black}\cite{safaryan2021stochastic} showed the convergence of {\scriptsize SIGN}SGD given the assumption that the element-wise probability of wrong aggregation is less than $\frac{1}{2}$.} In the presence of data heterogeneity, \cite{chen2019distributed} showed that the convergence of {\scriptsize SIGN}SGD is not guaranteed and proposed to add carefully designed noise to ensure a convergence rate of $O(d^{\frac{3}{4}}/T^{\frac{1}{4}})${\color{black}, while \cite{jin2020stochastic} further incorporated differential privacy}. In this work, we address the non-convergence issue of {\scriptsize SIGN}SGD with heterogeneous data by incorporating sparsification, which further improves communication efficiency.


\textbf{Sparsification:} The sparsification methods only keep a subset of the components of the gradients, resulting in a sparse representation of the original gradients. {\color{black}The Random-$k$ scheme} sparsifies the gradient vector by choosing $k$ components uniformly at random \cite{stich2018sparsified}, while {\color{black}the Top-$k$ scheme} selects $k$ components with the largest magnitudes \cite{alistarh2018convergence}. Similarly, {\color{black}the Threshold-$v$ scheme} selects the components whose absolute values are larger than a given threshold $v$ \cite{lin2018deep,sahu2021rethinking}. \cite{rothchild2020fetchsgd} incorporated sketching, which approximates the {\color{black}top-$k$} components of the gradients, into the FL framework.  \cite{wangni2018gradient} proposed an adaptive sparsification method to maximize the sparsity given a constraint on the induced variance. \cite{sattler2019robust,sattler2019sparse} proposed sparse ternary compression (STC), which combines Top-$k$ sparsification and binarization. {\color{black}However, error feedback is usually needed on the worker side to compensate for the bias introduced by sparsification.}


\textbf{Error-Compensated SGD:} \cite{seide20141} corrected the quantization error of 1-bit SGD by adding error feedback in subsequent updates and empirically achieved almost no accuracy loss. \cite{wu2018error} proposed the error-compensated SGD and proved its convergence for quadratic optimization with unbiased stochastic quantizations. \cite{stich2018sparsified} proved the convergence of the proposed error-compensated algorithm for strongly-convex loss functions, and \cite{alistarh2018convergence} established the convergence of sparsified gradient methods with error compensation. \cite{tang2019doublesqueeze} considered more general compressors with a bounded magnitude of error, with error feedback on both the server side and the worker side. \cite{liu2020double} proposed DORE, which compressed the gradient residual. \cite{basu2019qsparse} combined Top-$k$ and Random-$k$ sparsification with quantization and local update, along with error feedback. \cite{gao2021convergence} further considered bidirectional compression. \cite{philippenko2020bidirectional} proposed Artemis and provided a theoretical framework for bidirectional compression with heterogeneous data and partial worker participation, without considering the local update scheme. Similar to the error-feedback methods, they required the workers to keep track of the gradients.


{\color{black}To address the non-convergence issue of} sign-based methods, \cite{karimireddy2019error} proposed {\scriptsize EF\text{-}SIGN}SGD, which combined the error compensation methods and {\scriptsize SIGN}SGD; however, only the single worker scenario was considered. \cite{zheng2019communication} further extended it to the multi-worker scenario and the convergence was established. {\color{black}\cite{safaryan2021stochastic} incorporated the momentum mechanism on the worker side. Nonetheless, both the error-feedback scheme and the momentum mechanism require all the workers to participate in training, which may not be feasible in practice. In this work, the proposed methods do not require error feedback on the worker side, and therefore are compatible with the worker sampling mechanism in FL.} \cite{horvath2021better} introduced a new compressor to transform contractive compressors into unbiased ones, and therefore alleviated the need for the error-feedback mechanism for biased compressors. However, it cannot be applied to address the issue of {\scriptsize SIGN}SGD without incurring additional communication overhead.
\begin{algorithm}[th!]
\caption{Distributed Learning with Compressed Gradient and Worker Sampling}
\label{QuantizedSIGNSGD}
\begin{algorithmic}
\STATE \textbf{Input}: learning rate $\eta$, current hypothesis vector $w^{(t)}$, $M$ workers each with a local dataset $D_{m}$, the compressor $Q(\cdot)$ for workers, the compression budget $\boldsymbol{B}^{(t)}_{m}$ for each worker $m$, the aggregation rule $\mathcal{C}(\cdot)$ for the server.
\FOR{each communication round $t = 0,1,...,T-1$}
\STATE The server selects a random set of workers $S^{(t)}$
\FOR{Each worker $m \in S^{(t)}$}
\STATE Computes the local gradient $\boldsymbol{g}_{m}^{(t)}$
\STATE Sends $\Delta_{m}^{(t)} = Q(\boldsymbol{g}_{m}^{(t)},\boldsymbol{B}^{(t)}_{m})$ back to the server
\ENDFOR
\STATE The server pushes $\tilde{\boldsymbol{g}}^{(t)} = \mathcal{C}\big(\frac{1}{|S^{(t)}|}\sum_{m \in S^{(t)}}\Delta_{m}^{(t)})$ \textbf{to} all the workers.
\STATE All the workers update $w^{(t+1)} = w^{(t)} - \eta\tilde{\boldsymbol{g}}^{(t)}$
\ENDFOR
\end{algorithmic}
\end{algorithm}
\vspace{-0.2in}
\section{Problem Setup}
We consider a typical distributed learning paradigm, in which $M$ workers (denoted as $\mathcal{M}$) collaboratively learn a global model. Formally, the goal of the workers is to  solve the optimization problem of the form
\begin{equation}
\min_{w\in \mathbb{R}^d}F(w) \overset{\mathrm{def}}{=} \frac{1}{M}\sum_{m=1}^{M}f_{m}(w),
\end{equation}
in which $f_{m}:\mathbb{R}^{d}\rightarrow\mathbb{R}$ is the local objective function of worker $m$, which is defined by its local dataset, and $w$ is the model parameters to be optimized. For a supervised learning problem, there is a sample space $I = X \times Y$, where $X$ is a space of feature vectors and $Y$ is a label space. Given the hypothesis space $\mathcal{W}$ parameterized by $d$ dimensional vectors, we define a loss function $l: \mathcal{W}\times I \rightarrow \mathbb{R}$ that measures the loss of prediction on the data point $(x,y) \in I$ made with the hypothesis vector $w \in \mathcal{W}$. In this case, $f_{m}(w)$ is defined as the empirical risk below
\begin{equation}
f_{m}(w)=\frac{1}{|D_{m}|}\sum_{(x_n,y_n)\in D_{m}}l(w;(x_n,y_n)),
\end{equation}
where $|D_{m}|$ is the size of worker $m$'s local dataset $D_{m}$. When the local datasets are homogeneously distributed, we have $\mathbb{E}[f_{m}(w)]=F(w)$, where the expectation is over the training data distribution. In this work, we do not make any assumptions on the training data distribution, i.e., \textit{the training data can be arbitrarily distributed across the workers}.

We consider a parameter server paradigm as in federated learning \cite{mcmahan2017communication}, in which a parameter server coordinates the collaborative training of the $M$ workers. During each communication round $t$, the parameter server selects a subset of workers to perform local model training. Each selected worker samples a batch of training examples, based on which it computes the stochastic gradient $\boldsymbol{g}_{m}^{(t)}$. {\color{black}In this paper, gradient compression is adopted to improve communication efficiency. More specifically, instead of directly sharing the gradient $\boldsymbol{g}_{m}^{(t)}$, the workers share a compressed version $Q(\boldsymbol{g}_{m}^{(t)},\boldsymbol{B}_{m}^{(t)})$ with the parameter server, in which $Q(\cdot)$ is the compressor and $\boldsymbol{B}_{m}^{(t)}$ is the parameter that specifies the compression budget (discussed in detail in Section \ref{ternarysection}). After receiving the gradients from the selected workers, the parameter server aggregates the collected gradients with some aggregation rule $\mathcal{C}(\cdot)$ and sends the aggregated gradient back to all the workers. Finally, the workers update their local model parameters using the aggregated gradient. The corresponding algorithm is presented in Algorithm \ref{QuantizedSIGNSGD}.}

For {\scriptsize SIGN}SGD with majority vote \cite{bernstein2018signsgd1}, $Q(\cdot)=\mathcal{C}(\cdot)=sign(\cdot)$ and $S^{(t)} =\mathcal{M}$. Intuitively, when the local training data are homogeneous across the workers, the gradients of the workers follow the same distribution. Therefore, the workers are supposed to share the same magnitude, which can be discarded without ruining the convergence guarantees. In the heterogeneous data distribution scenario, however, the magnitude information captures the data heterogeneity, which cannot be discarded.

\section{The Convergence of Ternary-based Gradient Descent Methods}\label{ternarysection}
\noindent In this section, we extend the convergence analysis for sign-based gradient descent to the ternary-based gradient descent methods. More specifically, in Algorithm \ref{QuantizedSIGNSGD}, $Q(\boldsymbol{g}_{m,i}^{(t)},\boldsymbol{B}^{(t)}_{m}) \in \{-1,0,1\}$ and $\mathcal{C}(\cdot) = sign(\cdot)$. Compared to $sign(\cdot)$, $Q(\boldsymbol{g}_{m,i}^{(t)},\boldsymbol{B}^{(t)}_{m})$ further captures two scenarios: 1) gradient sparsification and 2) worker sampling, in which $Q(\boldsymbol{g}_{m,i}^{(t)},\boldsymbol{B}^{(t)}_{m}) = 0$ if the $i$-th coordinate of worker $m$'s gradient is zeroed out or worker $m$ is not selected. For ternary-based methods, we have the following theorem.

\begin{theorem}\label{Theorem1}
Let $u_{1},u_{2},\cdots,u_{M}$ be $M$ known and fixed real numbers with $\frac{1}{M}\sum_{m=1}^{M}u_{m} \neq 0$, and consider random variables $\hat{u}_{m} \in \{-1,0,1\}$ {\color{black}(which is the compressed version of $u_{m}$)}, $1\leq m \leq M$. Denote $P(\hat{u}_{m} = -sign(\frac{1}{M}\sum_{m=1}^{M}u_{m}))=p_{m}$, $P(\hat{u}_{m} = sign(\frac{1}{M}\sum_{m=1}^{M}u_{m}))=q_{m}$ and $P(\hat{u}_{m} = 0)=1-p_{m}-q_{m}$. Let $\Bar{p} = \frac{1}{M}\sum_{m=1}^{M}p_{m}$ and $\Bar{q} = \frac{1}{M}\sum_{m=1}^{M}q_{m}$. If $\Bar{q} > \Bar{p}$, {\color{black}the probability of wrong aggregation is given by}
\begin{equation}\label{ProbabilityOfError}
\begin{split}
P\bigg(sign\bigg(\frac{1}{M}\sum_{m=1}^{M}\hat{u}_{m}\bigg)&\neq sign\bigg(\frac{1}{M}\sum_{m=1}^{M}u_{m}\bigg)\bigg) \\
&\leq [1-(\sqrt{\Bar{q}}-\sqrt{\Bar{p}})^2]^{M}.
\end{split}
\end{equation}
\end{theorem}
\begin{Remark}
Theorem \ref{Theorem1} considers the scalar case, which can be readily generalized to the vector case by applying the result to each coordinate independently. More specifically, $sign(\frac{1}{M}\sum_{m=1}^{M}u_{m})$ can be understood as the sign of the global gradient that we want to recover, $\hat{u}_{m}$ is the compressed gradient shared by worker $m$, and $sign(\frac{1}{M}\sum_{m=1}^{M}\hat{u}_{m})$ is the aggregated result. Theorem \ref{Theorem1} suggests that when the average probability of wrong signs (i.e., $\bar{p}$) is less than that of correct signs (i.e., $\bar{q}$), there exists some $M$ such that the probability of wrong aggregation is less than $\frac{1}{2}$, based on which the convergence can be established (c.f. Theorem \ref{convergerate}). {\color{black}Intuitively, suppose that $sign(\frac{1}{M}\sum_{m=1}^{M}u_{m})=1$, $\sum_{m=1}^{M}\hat{u}_{m}$ converges in distribution to a Gaussian random variable with mean $M(\Bar{q}-\Bar{p}) > 0$ according to the Lyapunov central limit theorem. Therefore, the probability of wrong aggregation is strictly smaller than $\frac{1}{2}$. A similar argument holds for $sign(\frac{1}{M}\sum_{m=1}^{M}u_{m})=-1$.}

For {\scriptsize SIGN}SGD with large batch size and homogeneous data distribution, $\sqrt{\bar{q}}-\sqrt{\bar{p}}$ approaches $1$, and therefore the probability of wrong aggregation is small. Without any assumptions on the data distribution, however, $\bar{q} > \bar{p}$ is not guaranteed for the deterministic $sign(\cdot)$.
\end{Remark}

Given Theorem \ref{Theorem1}, we endeavor to propose a sparsification scheme such that $\Bar{q} > \Bar{p}$ for an \textit{arbitrary} realization of $u_{m}$'s, which is defined as follows.
\begin{Definition}\label{definitioncompressor}[\textbf{The Proposed Compressor}]
Given a gradient vector $\boldsymbol{g}_{m}$, the proposed compressor outputs
\begin{equation}\label{DefinitionofSparsification3}
sparsign(\boldsymbol{g}_{m,i}^{(t)},\boldsymbol{B}^{(t)}_{m}) =
\begin{cases}
sign(\boldsymbol{g}_{m,i}^{(t)}), \hfill ~~~\text{w.p. $|\boldsymbol{g}_{m,i}^{(t)}|\boldsymbol{B}^{(t)}_{m,i}$}, \\
0, \hfill ~~~\text{w.p. $1-|\boldsymbol{g}_{m,i}^{(t)}|\boldsymbol{B}^{(t)}_{m,i}$},
\end{cases}
\end{equation}
in which $\boldsymbol{g}_{m,i}^{(t)}$ is the $i$-th entry of $\boldsymbol{g}_{m}$ and $\boldsymbol{B}^{(t)}_{m,i} \leq \frac{1}{|\boldsymbol{g}_{m,i}^{(t)}|}$ controls the sparsity budget. More specifically, the expected number of non-zero entries is given by $\sum_{i=1}^{d}|\boldsymbol{g}_{m,i}^{(t)}|\boldsymbol{B}^{(t)}_{m,i}$.
\end{Definition}

\begin{Remark}
\textbf{Relation to \textit{TernGrad} \cite{wen2017terngrad} and \textit{1-bit QSGD} \cite{alistarh2017qsgd}:} the compressors in \textit{TernGrad} and 1-bit QSGD can be understood as scaled versions of $sparsign(\cdot)$ with $\boldsymbol{B}^{(t)}_{m,i} = \frac{1}{\max_{m}||\boldsymbol{g}_{m}||_{\infty}}$ and $\boldsymbol{B}^{(t)}_{m,i} = \frac{1}{||\boldsymbol{g}_{m}||_{2}}$, $\forall i$, respectively. In addition, the output of the compressors in \textit{TernGrad} and \textit{1-bit QSGD} are further scaled by the factors $\max_{m}||\boldsymbol{g}_{m}||_{\infty}$ and $||\boldsymbol{g}_{m}||_{2}$, respectively, to preserve the unbiasedness. Algorithm \ref{QuantizedSIGNSGD} is different from \textit{TernGrad} and \textit{1-bit QSGD} in four aspects:  1) The aggregated gradient $\tilde{\boldsymbol{g}}^{(t)}$ in Algorithm \ref{QuantizedSIGNSGD} is a biased estimate of the true gradient. As a result, the convergence analysis of Algorithm \ref{QuantizedSIGNSGD} is completely different from that of \textit{TernGrad} and \textit{1-bit QSGD}. 2) It enjoys flexibility in adjusting the compression ratio by changing $\boldsymbol{B}^{(t)}_{m,i}$'s. 3) It enables compressed communication from the server to the workers. {\color{black}More specifically, for majority vote (i.e., $\mathcal{C}(\cdot) = sign(\cdot)$), there is a reduction of $32\times$ in communication overhead.} 4) It does not require the exchange of $||\boldsymbol{g}_{m}||_{\infty}$ and $||\boldsymbol{g}_{m}||_{2}$, and therefore, is robust against re-scaling attacks that manipulate the magnitudes \cite{jin2020stochastic}. 
\end{Remark}

\begin{Corollary}\label{Corollary1}
Given the same $u_{1},u_{2},\cdots,u_{M}$ as in Theorem \ref{Theorem1} and consider random variables $\hat{u}_{m}=sparsign(u_{m},B_{m})$, $1\leq m \leq M$. Let $\mathcal{A}$ and $\mathcal{A}^{c}$ denote the sets of workers such that $u_{m} \neq sign(\frac{1}{M}\sum_{m=1}^{M}u_{m}), \forall m \in \mathcal{A}$ and $u_{m} = sign(\frac{1}{M}\sum_{m=1}^{M}u_{m}), \forall m \in \mathcal{A}^{c}$, respectively. Let $S$ denote the set of selected workers, then $\Bar{p} = \frac{1}{M}\sum_{m \in \mathcal{A}}|u_{m}B_{m}P(m \in S)|$ and $\Bar{q} = \frac{1}{M}\sum_{m \in \mathcal{A}^{c}}|u_{m}B_{m}P(m \in S)|$ in (\ref{ProbabilityOfError}).
\end{Corollary}

\begin{Remark}\label{Remark3}
When $B_{m} = B$ and $\Pr(m \in S) = p_{s}, \forall m$, i.e., the workers are selected uniformly at random and share the same $B_{m}$, we have $\Bar{p} = \frac{1}{M}\sum_{m \in \mathcal{A}}|u_{m}|Bp_{s}$, $\Bar{q} = \frac{1}{M}\sum_{m \in \mathcal{A}^{c}}|u_{m}|Bp_{s}$, which ensures $\Bar{q} > \Bar{p}$. That being said, there always exists some $M$ such that the probability of wrong aggregation $P\left(sign\left(\sum_{m \in S^{(t)}}\hat{u}_{m}\right)\neq sign\left(\sum_{m=1}^{M}u_{m}\right)\right) \leq [1-Bp_{s}(\sqrt{\frac{1}{M}\sum_{m \in \mathcal{A}^{c}}|u_{m}|}-\sqrt{\frac{1}{M}\sum_{m \in \mathcal{A}}|u_{m}|})^2]^{M} < \frac{1}{2}$. Moreover, in this case, the right-hand side of (\ref{ProbabilityOfError}) depends on the worker sampling probability $p_s$, i.e., a larger $p_{s}$ suggests a smaller probability of wrong aggregation.
\end{Remark}
Given the above results, in the following, we show the convergence of Algorithm \ref{QuantizedSIGNSGD}. In order to facilitate the convergence analysis, the following commonly adopted assumptions are made.
\begin{Assumption}\label{A1}(Lower Bound).
 For all $w$ and some constant $F^{*}$, we have objective value $F(w) \geq F^{*}$.
\end{Assumption}
\begin{Assumption}\label{A2}(Smoothness).
$\forall w_1,w_2$, we require for some non-negative constant $L$
\begin{equation}
F(w_1) \leq F(w_2) + \langle\nabla F(w_2), w_1-w_2\rangle + \frac{L}{2}||w_1 - w_2||^{2}_2,
\end{equation}
where $\langle\cdot,\cdot\rangle$ is the standard inner product.
\end{Assumption}

\begin{Assumption}\label{A3}(Unbiased Local Estimator)
For any $w \in \mathbb{R}^{d}$, each worker has access to an independent and unbiased estimator $\boldsymbol{g}_{m}$ of the true local gradient $\nabla f_{m}(w)$, i.e., $\mathbb{E}[\boldsymbol{g}_{m}]=\nabla f_{m}(w)$.
\end{Assumption}

\begin{Assumption}\label{A4}(SPB: Success Probability Bounds).
For any $w \in \mathbb{R}^{d}$, the workers have access to estimator $\frac{1}{M}\sum_{m=1}^{M}\boldsymbol{g}_{m}$ of the true {\color{black}global} gradient $\frac{1}{M}\sum_{m=1}^{M}\nabla f_{m}(w)$ such that
\begin{equation}
\begin{split}
P\bigg(sign\left(\frac{1}{M}\sum_{m=1}^{M}\boldsymbol{g}_{m,i}\right) &\neq sign\bigg(\frac{1}{M}\sum_{m=1}^{M}\nabla f_{m}(w)_{i}\bigg)\bigg) \\
&\triangleq \rho_{i} < \frac{1}{2}, \forall i,
\end{split}
\end{equation}
in which $\boldsymbol{g}_{m,i}$ and $\nabla f_{m}(w)_{i}$ are the $i$-th entry of $\boldsymbol{g}_{m}$ and $\nabla f_{m}(w)$, respectively.
\end{Assumption}

We note that Assumptions \ref{A1}-\ref{A3} are commonly adopted in the literature. Assumption \ref{A4} is introduced in \cite{safaryan2021stochastic}, which holds under mild assumptions on gradient noise.
\begin{Remark}
Different from \cite{safaryan2021stochastic} that assumes SPB on each individual worker (and therefore cannot deal with heterogeneous data), Assumption \ref{A4} is a relaxed version that concerns the average of the local gradients from the workers. More specifically, it can be verified that if the training data sampling noise in each coordinate of the stochastic gradient is unimodal and symmetric about the mean (e.g., Gaussian) \cite{bernstein2018signsgd1}, Assumption \ref{A4} holds for an arbitrary realization of $\nabla f_{m}(w)$'s.
\end{Remark}

\begin{theorem}\label{convergerate}
Suppose Assumptions \ref{A1}, \ref{A2} and \ref{A4} are satisfied, and the learning rate is set as $\eta=\frac{1}{\sqrt{Td}}$. Then by running Algorithm \ref{QuantizedSIGNSGD} (termed {\scriptsize SPARSIGN}SGD) with $Q(\boldsymbol{g}_{m}^{(t)},\boldsymbol{B}^{(t)}_{m}) = sparsign(\boldsymbol{g}_{m}^{(t)},\boldsymbol{B}^{(t)}_{m})$, where $\boldsymbol{B}_{m,i}^{(t)} = B_{i}^{(t)}$, $\forall m$, and $\mathcal{C}(\cdot) = sign(\cdot)$ for $T$ iterations, we have
\begin{equation}
\color{black}
\begin{split}
&\frac{1}{T}\sum_{t=1}^{T}\sum_{i=1}^{d}(1-2\rho_{i})(1-2\kappa_{i}^{(t)})|\nabla F(w^{(t)})_{i}| \\
&\leq \frac{\mathbb{E}[F(w^{(0)}) - F(w^{(T+1)})]\sqrt{d}}{\sqrt{T}} + \frac{L\sqrt{d}}{2\sqrt{T}}\\
&\leq\frac{(F(w^{(0)}) - F^{*})\sqrt{d}}{\sqrt{T}} + \frac{L\sqrt{d}}{2\sqrt{T}},
\end{split}
\end{equation}
where
\begin{equation}\nonumber
\begin{split}
&\kappa_{i}^{(t)} = \mathbb{E}\bigg[\bigg[1-B_{i}^{(t)}p_{s}\times\\
&\bigg(\frac{|\frac{1}{M}\sum_{m=1}^{M}\boldsymbol{g}_{m,i}^{(t)}|}{\sqrt{\frac{1}{M}\sum_{m \in \mathcal{A}^{c}_{(t)}}|\boldsymbol{g}_{m,i}^{(t)}|}+\sqrt{\frac{1}{M}\sum_{m \in \mathcal{A}_{(t)}}|\boldsymbol{g}_{m,i}^{(t)}|}}\bigg)^2\bigg]^{M}\bigg],
\end{split}
\end{equation}
$\mathcal{A}_{(t)}$ and $\mathcal{A}^{c}_{(t)}$ are the set of workers such that $sign(\boldsymbol{g}_{m,i}^{(t)}) \neq sign(\frac{1}{M}\sum_{m=1}^{M}\boldsymbol{g}_{m,i}^{(t)})$ and $sign(\boldsymbol{g}_{m,i}^{(t)}) = sign(\frac{1}{M}\sum_{m=1}^{M}\boldsymbol{g}_{m,i}^{(t)})$, respectively, and the expectation is over the randomness of the gradients $\boldsymbol{g}_{m}^{(t)}$.
\end{theorem}
\begin{Remark}
As long as $P(|\frac{1}{M}\sum_{m=1}^{M}\boldsymbol{g}_{m,i}^{(t)}|=0) < 1$, there exists some $M$ such that $\kappa_{i}^{(t)} < \frac{1}{2}$. In addition, it can be observed that, if $B_{i}^{(t)}$ and $p_{s}$ are independent of $M$, $\lim_{m \rightarrow \infty}\kappa_{i}^{(t)} = 0$, which recovers the convergence rate of {\scriptsize SIGN}SGD with majority {\color{black}vote} in the centralized scenario \cite{safaryan2021stochastic}. That being said, the impact of data heterogeneity is mitigated in the large $M$ regime. In the ideal case where all the workers share the same gradients, we can set $B_{i}^{(t)} = |\boldsymbol{g}_{m,i}^{(t)}|, \forall i$. As a result, we have $\kappa_{i}^{(t)} = 0$, $\forall i$. Moreover, it can be observed that in the full gradient descent case (i.e., $\boldsymbol{g}_{m}^{(t)} = \nabla f_{m}(w^{(t)})$), as $\frac{1}{M}\sum_{m=1}^{M}\nabla f_{m}(w^{(t)})_{i}$ decreases, $\kappa_{i}^{(t)}$ approaches 1. In this case, Algorithm \ref{QuantizedSIGNSGD} does not converge to the (local) optimum. Nonetheless, for SGD, we have $P(|\frac{1}{M}\sum_{m=1}^{M}\boldsymbol{g}_{m,i}^{(t)}|=0) < 1$, and therefore $\kappa_{i}^{(t)}$ does not approach 1. This suggests that the stochasticity in SGD is beneficial.
\end{Remark}


\begin{algorithm}[th!]
\caption{{\scriptsize EF-SPARSIGN}SGD with Local Updates}
\label{QuantizedSIGNSGDLocal}
\begin{algorithmic}
\STATE \textbf{Input}: learning rate $\eta$, current hypothesis vector $w^{(t)}$, $M$ workers each with a local dataset $D_{m}$, the compression budgets $\boldsymbol{B}_{m}^{(t,c)}$ and $\boldsymbol{B}^{(t)}_{m}$ for each worker $m$, the aggregation rule $\mathcal{C}(\cdot)$ for the server, residual error vector $\tilde{\boldsymbol{e}}^{(t)}$.
\FOR{each communication round $t = 0,1,...,T-1$}
\STATE Server selects a random set of workers $S^{(t)}$
\FOR{Each worker $m \in S^{(t)}$}
\STATE Set $w^{(t, 0)}_{m} = w^{(t)}$
\FOR{Each local step $c=0,1,...,\tau-1$}
\STATE $w^{(t, c+1)}_{m} = w^{(t, c)}_{m} - \eta_{L}Q(\boldsymbol{g}_{m}^{(t,c)},\boldsymbol{B}_{m}^{(t,c)})$
\ENDFOR
\STATE Sends $\Delta_{m}^{(t)} = Q(\sum_{c=0}^{\tau-1}Q(\boldsymbol{g}_{m}^{(t,c)},\boldsymbol{B}_{m}^{(t,c)}),\boldsymbol{B}^{(t)}_{m})$ to the server
\ENDFOR
\STATE Server pushes $\tilde{\boldsymbol{g}}^{(t)} = \mathcal{C}\big(\frac{1}{|S^{(t)}|}\sum_{m \in S^{(t)}}\Delta_{m}^{(t)}+\tilde{\boldsymbol{e}}^{(t)}\big)$ \textbf{to} all the workers.
\STATE Server updates residual error
\begin{equation}\label{residualupdate}
\tilde{\boldsymbol{e}}^{(t+1)} = \frac{1}{|S^{(t)}|}\sum_{m \in S^{(t)}}\Delta_{m}^{(t)} + \tilde{\boldsymbol{e}}^{(t)} - \tilde{\boldsymbol{g}}^{(t)}.
\end{equation}
\STATE All workers update $w^{(t+1)} = w^{(t)} - \eta\eta_{L}\tilde{\boldsymbol{g}}^{(t)}$
\ENDFOR
\end{algorithmic}
\end{algorithm}
\vspace{-0.1in}
\section{{\scriptsize SPARSIGN}SGD with Local Updates}\label{sparse_with_local}
\noindent Besides gradient compression, another popular approach to improve the communication efficiency is the local update scheme, which reduces the number of communication rounds needed for convergence \cite{mcmahan2017communication}. The basic idea is to allow the workers to perform multiple local iterations during each communication round.

In this section, we incorporate the local update scheme into {\scriptsize SPARSIGN}SGD, and the corresponding algorithm is presented in Algorithm \ref{QuantizedSIGNSGDLocal}. More specifically, during each communication round, the parameter server selects a random set of workers, and each worker performs multiple local steps using the compressed gradients $Q(\boldsymbol{g}_{m}^{(t,c)},\boldsymbol{B}_{m}^{(t,c)})=sparsign(\boldsymbol{g}_{m}^{(t,c)},\boldsymbol{B}_{m}^{(t,c)})$. After $\tau$ local steps, each worker transmits the compressed model update to the parameter server. The parameter server adopts an $\alpha$-approximate compressor $\mathcal{C}(\cdot)$ (i.e., $||\mathcal{C}(\boldsymbol{x})-\boldsymbol{x}||_{2}^{2} \leq (1-\alpha)||\boldsymbol{x}||_{2}^{2}, \forall \boldsymbol{x}$ \cite{karimireddy2019error}) to aggregate the model updates and utilizes the error-feedback mechanism to compensate for the error introduced by the server-to-worker compression. In our experiments, we adopt the scaled sign compressor {\color{black}(i.e., $\mathcal{C}(\boldsymbol{x}) = \frac{||\boldsymbol{x}||_{1}}{d}sign(\boldsymbol{x})$) \cite{karimireddy2019error}. We note that different from the majority of existing works, the error-feedback mechanism is only adopted on the server side. Therefore, the proposed method is compatible with worker sampling in federated learning}.

Different from the existing literature that uses full precision gradients for local updates, Algorithm \ref{QuantizedSIGNSGDLocal} adopts compression during the local training as well. On one hand, this mimics the gradient compression at the server side, especially when $\mathcal{C}(\cdot) \in \{-1,0,1\}$. On the other hand, the output of $sparsign$ lies in the ternary set $\{-1,0,1\}$, and therefore is itself bounded by definition. In this sense, the difference between the global model and the local models after local training is also bounded, which further alleviates the impact of data heterogeneity. As a result, different from the existing works in federated learning, e.g., \cite{li2018federated,haddadpour2021federated}, we do not need the bounded gradient dissimilarity assumption for convergence analysis. Formally, we have the following theorem.

\begin{theorem}\label{EFDPSIGNConvergence2}
When Assumptions \ref{A1}, \ref{A2} and \ref{A3} are satisfied, by running Algorithm \ref{QuantizedSIGNSGDLocal} with $\eta_{L} = \frac{1}{\sqrt{Td}\tau}$, $\eta = \tau$, $\boldsymbol{B}_{m}^{(t,c)} = B_{l}^{(t)}\cdot\boldsymbol{1}$, $\boldsymbol{B}^{(t)}_{m} = B_{g}^{(t)}\cdot\boldsymbol{1}~\forall m$, $B_{l}^{(t)}B_{g}^{(t)} = B$, we have
\begin{equation}
\begin{split}
\frac{1}{T}\sum_{t=0}^{T-1}||\nabla F(w^{(t)})||^2 \leq \frac{(F(w^{(0)})-F^{*})\sqrt{d}}{B\tau\sqrt{T}} \\
+ \frac{(1+L+L^2\beta)\sqrt{d}}{B\tau\sqrt{T}}+ \frac{L^2(\tau+1)(2\tau+1)}{6T\tau^2}.
\end{split}
\end{equation}
\end{theorem}

\begin{Remark}
We note that different from Theorem \ref{convergerate}, Theorem \ref{EFDPSIGNConvergence2} does not require a large $M$.
\end{Remark}
\begin{Remark}
\color{black}
In the implementation of the $sparsign$ compressor, there are multiple ways to set the compression budgets $B_{l}^{(t)}$ and $B_{g}^{(t)}$. For instance, the magnitude sharing protocol in \cite{wen2017terngrad} can be adopted, in which the workers share the $L_{\infty}$ norm of the gradients, and the parameter server sets the compression budget by taking the maximum of the $L_{\infty}$ norms. In our experiments, we set some pre-determined values for $B_{l}^{(t)}$ and $B_{g}^{(t)}$. In this case, it is possible that the probabilities in Definition \ref{definitioncompressor} fall out of $[0,1]$, and we round them to 0 or 1, respectively. This is equivalent to gradient clipping, which is commonly adopted in deep learning \cite{chen2020understanding,zhang2022understanding}.
\end{Remark}
\section{Experiments}
\noindent In this section, we first verify the theoretical results in Section \ref{ternarysection} and demonstrate the effectiveness of the proposed compressor by comparing the probability of wrong aggregation between the deterministic $sign$ and $sparsign$ in the minimization of the Rosenbrock function. Then, we examine the performance of the proposed algorithms on Fashion-MNIST, CIFAR-10 and CIFAR-100.

\begin{table*}[t]
\vspace{-0.1in}
\caption{Learning Performance on Fashion-MNIST ($\alpha = 0.1$)}
\vspace{-0.1in}
\label{table_fashionmnist}
\begin{center}
\begin{sc}
\begin{tabular}{cccccc}
\toprule
Algorithm & \makecell{Final Accuracy} & \makecell{Communications Round \\to Achieve 74\%} & \makecell{Communication Overhead \\to Achieve 74\% (bits)}\\
\midrule
\makecell{{\scriptsize SIGN}SGD} & 74.44$\pm$0.71\%& 193 & $4.56\times10^{7}$\\
\makecell{Scaled {\scriptsize SIGN}SGD} & 69.61$\pm$1.99\%& N.A. & N.A.\\
\makecell{Noisy {\scriptsize SIGN}SGD} & 77.84$\pm$0.37\%& 79 & $1.88\times10^{7}$\\
\makecell{1-bit $L_2$ norm QSGD} & 79.05$\pm$1.22\%& 75 & $1.98\times10^{5}$\\
\makecell{1-bit $L_{\infty}$ norm QSGD} & 80.07$\pm$0.75\%& 68 & $1.13\times10^{6}$\\
\makecell{TernGrad} & 79.17$\pm$1.41\%& 66 & $4.34\times10^{5}$\\
\makecell{{\scriptsize SPARSIGN}SGD ($B=1$)} & 79.05$\pm$0.39\%& 65 & $8.19\times10^{5}$\\
\makecell{{\scriptsize EF-SPARSIGN}SGD \\($B_{l}^{(t)}=10$, $B_{g}^{(t)}=1$, $\tau = 1$)} & 80.75$\pm$0.20\%& 65 & $1.93\times10^{5}$\\
\bottomrule
\end{tabular}
\end{sc}
\end{center}
\end{table*}

\begin{figure}[t]
\centering
\begin{minipage}[t]{0.49\linewidth}
{\includegraphics[width=1\textwidth]{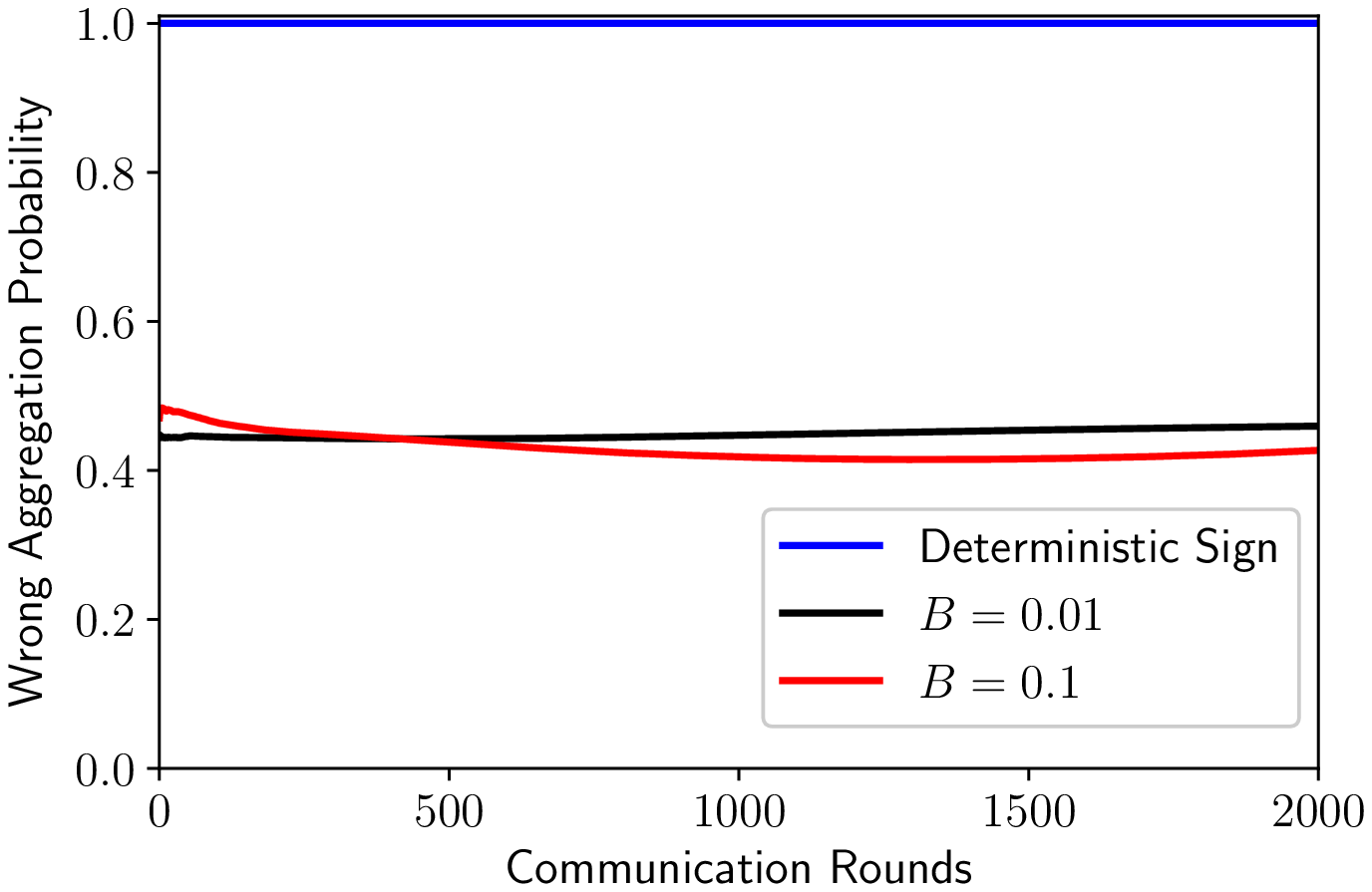}}
\end{minipage}
\begin{minipage}[t]{0.49\linewidth}
{\includegraphics[width=1\textwidth]{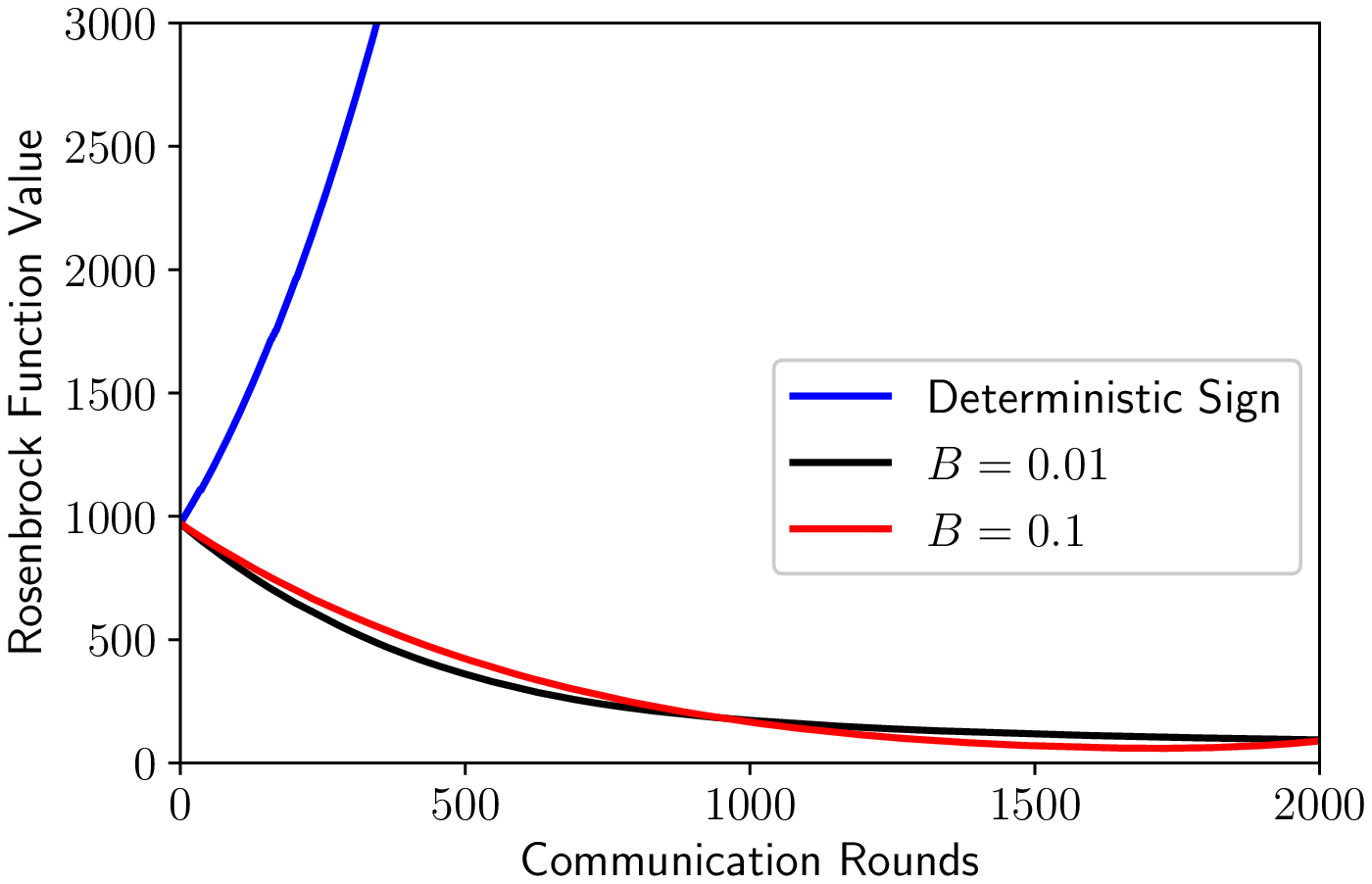}}
\end{minipage}%
\vspace{-0.1in}
\caption{The left figure shows the probability of wrong aggregation and the right figure shows the Rosenbrock function value that is minimized. During each communication round, 10 out of 100 workers are randomly selected to participate in the training.}
\label{suppimpact_pbar}
\vspace{-0.1in}
\end{figure}
\begin{figure}[t]
\centering
\begin{minipage}[t]{0.49\linewidth}
{\includegraphics[width=1\textwidth]{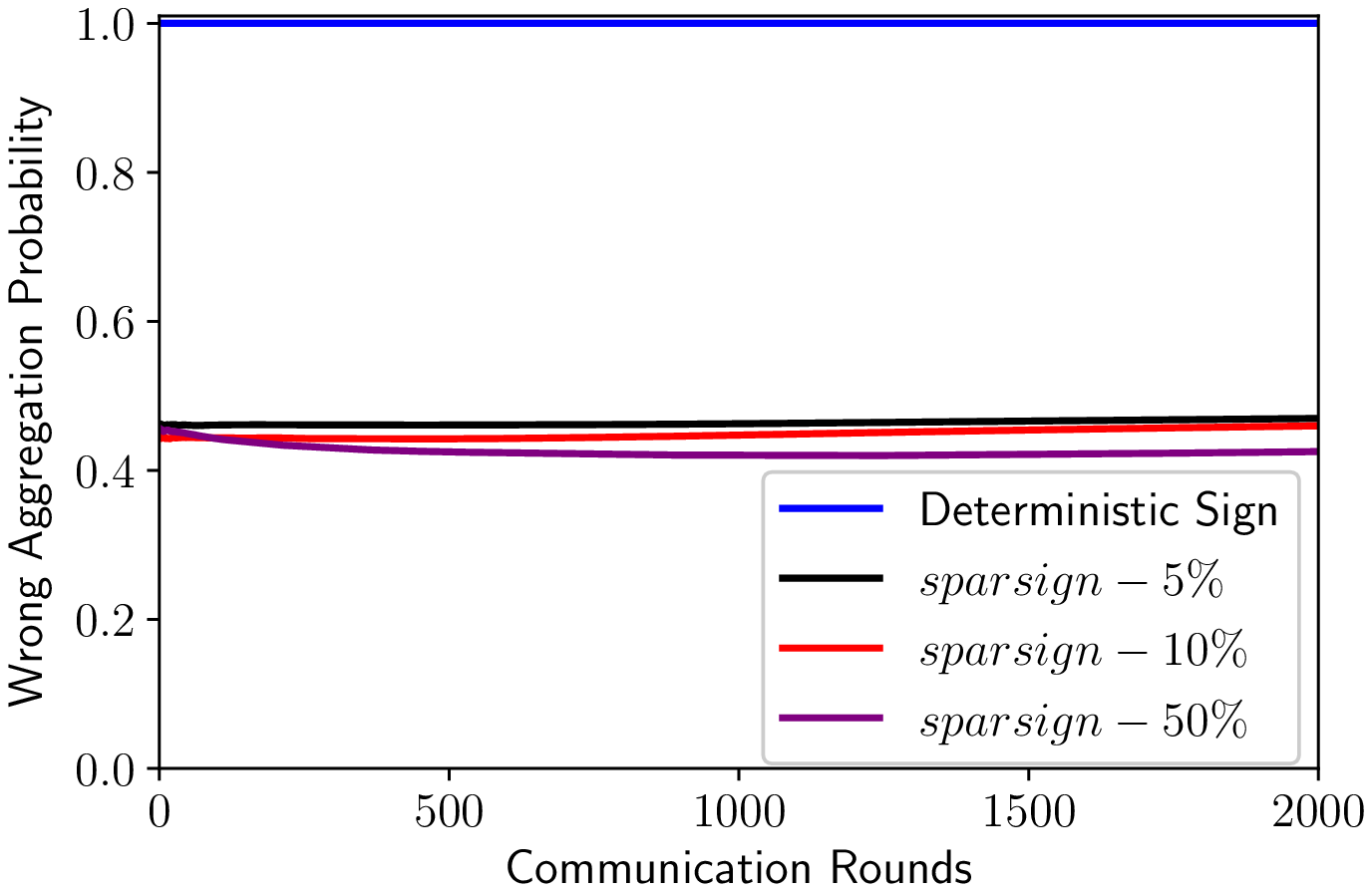}}
\end{minipage}
\begin{minipage}[t]{0.49\linewidth}
{\includegraphics[width=1\textwidth]{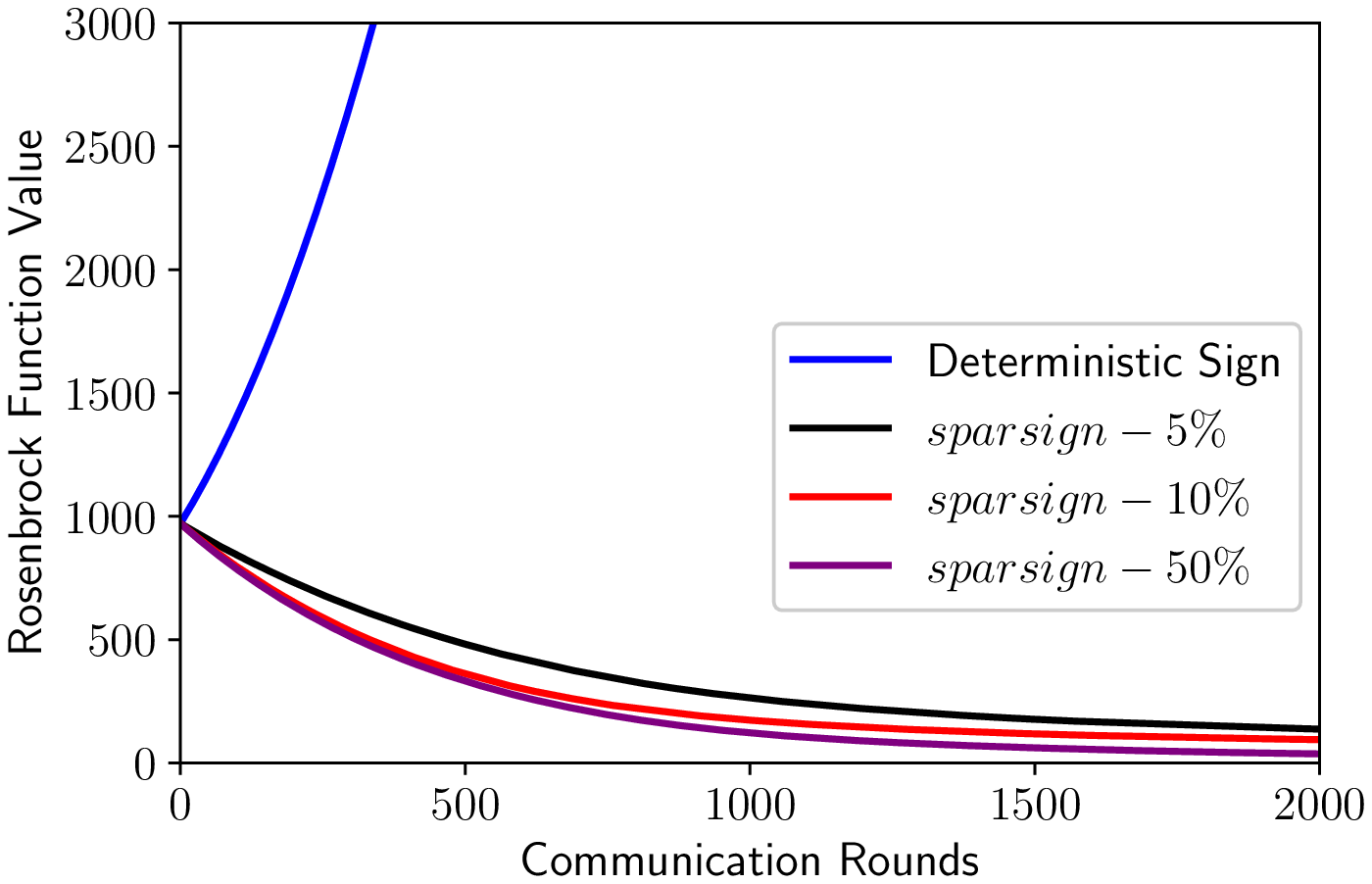}}
\end{minipage}%
\vspace{-0.1in}
\caption{The left figure shows the probability of wrong aggregation and the right figure shows the Rosenbrock function value that is minimized. For ``Deterministic Sign``, all the workers are selected to participate in the training during each communication round. For $sparsign$, $B = 0.01$, and $5\%$, $10\%$, $50\%$ workers are selected during each communication round, respectively.}
\label{suppimpact_M}
\vspace{-0.2in}
\end{figure}
\subsection{Minimization of the Rosenbrock Function}
\noindent In this subsection, we consider the minimization of the well-known Rosenbrock function with 10 variables as in \cite{safaryan2021stochastic}:
\begin{equation}
\begin{split}
    &~~~~~~~~~~~~~~~~~~F(x) = \sum_{i=1}^{d}F_{i}(x), \\
    &\text{where}~ F_{i}(x) = 100(x_{i+1}-x_{i}^2)+(1-x_{i})^2.
\end{split}
\end{equation}

The collaborative training of $M=100$ workers is considered. To simulate the data heterogeneity, we assume that each worker $m$ has access to a scaled objective $v_{m}F(\cdot)$, in which $v_{m}$ is some random number such that
\begin{equation}
\begin{split}
\sum_{m=1}^{M}v_{m} = 1,~~\sum_{m=1}^{M}\mathds{1}_{v_{m} < 0} = 80.
\end{split}
\end{equation}
The second condition suggests that the signs of the gradients from 80 out of the 100 workers are opposite to those of the true gradients. Fig. \ref{suppimpact_pbar} shows the probability of wrong aggregation and the Rosenbrock function value for deterministic sign (i.e., {\scriptsize SIGN}SGD) and $sparsign$ (i.e., {\scriptsize SPARSIGN}SGD) with $B \in \{0.01,0.1\}$. It can be observed that the probability of wrong aggregation {\color{black}(i.e., the signs of the aggregated results are different from those of the true gradients)} for the deterministic sign is 1, and {\scriptsize SIGN}SGD diverges, while the probability of wrong aggregation for $sparsign$ is always smaller than $\frac{1}{2}$, and {\scriptsize SPARSIGN}SGD converges, which verifies the effectiveness of $sparsign$. Fig. \ref{suppimpact_M} examines the impact of worker sampling. It can be observed that as the number of selected workers increases, the probability of wrong aggregation decreases, and the algorithm converges faster, which validates our results in Remark \ref{Remark3}.

\subsection{Results on Fashion-MNIST and CIFAR-10}
\noindent In this subsection, we implement our proposed methods with a three-layer fully connected neural network on the Fashion-MNIST dataset and VGG-9 \citep{lee2020enabling} on the CIFAR-10 dataset. For Fashion-MNIST, we use a fixed learning rate, which is tuned from the set $\{0.0001, 0.001, 0.01, 0.1, 1.0\}$. For CIFAR-10, we tune the initial learning rate from the set $\{0.0001, 0.001, 0.01, 0.1, 1.0\}$, which is reduced by a factor of 2 at communication round 1,500. We consider a scenario of $M=100$ normal workers and follow \cite{hsu2019measuring} to simulate heterogeneous data distribution, in which the training data on each worker are drawn independently with class labels following a Dirichlet distribution. More specifically, a vector of length $C$ that follows the Dirichlet distribution $Dir(\alpha)$ is generated, in which $C$ is the number of classes, and $\alpha$ controls the level of data heterogeneity. Each element of the vector specifies the {\color{black}proportion} of training examples that belong to the corresponding class. The mini-batch sizes of the workers are set to 128 and 32 for Fashion-MNIST and CIFAR-10, respectively. For the experiments on Fashion-MNIST, the test accuracy is evaluated for every communication round, while for those on CIFAR-10, it is evaluated for every 25 communication rounds.

We compare the proposed methods with various baseline {\color{black}compressors}, including {\scriptsize SIGN}SGD \cite{bernstein2018signsgd1}, {Scaled {\scriptsize SIGN}SGD} \cite{karimireddy2019error}, {Noisy {\scriptsize SIGN}SGD} \cite{chen2019distributed}, TernGrad \cite{wen2017terngrad}, QSGD \cite{alistarh2017qsgd} (including 1-bit $L_2$ norm QSGD and 1-bit $L_{\infty}$ norm QSGD). More details can be found in Section \ref{DetailBaselines} in the supplementary material. {\color{black}Particularly, {Scaled {\scriptsize SIGN}SGD} and {Noisy {\scriptsize SIGN}SGD} are proposed to address the non-convergence issue of {\scriptsize SIGN}SGD, while the compressor of QSGD is one of the most commonly adopted quantization methods in the literature (e.g., \citep{pmlr-v108-reisizadeh20a,haddadpour2021federated,philippenko2020bidirectional}).}

\begin{table*}[th!]
\vspace{-0.1in}
\caption{Learning Performance on CIFAR-10 ($\alpha = 0.5$)}
\vspace{-0.1in}
\label{table_cifar}
\begin{center}
\begin{sc}
\begin{tabular}{cccccc}
\toprule
Algorithm & \makecell{Final Accuracy} & \makecell{Communications Rounds \\to Achieve 55\%/74\%} & \makecell{Communication Overhead \\to Achieve 55\%/74\% (bits)}\\
\midrule
\makecell{{\scriptsize SIGN}SGD} & 55.35$\pm$0.71\%& 3,000/N.A. & $1.15\times10^{10}$/N.A.\\
\makecell{Scaled {\scriptsize SIGN}SGD} & 46.86$\pm$2.72\%& N.A./N.A. & N.A./N.A.\\
\makecell{Noisy {\scriptsize SIGN}SGD} & 74.41$\pm$0.61\%& 625/2,600 & $2.31\times10^{9}$/$9.89\times10^{9}$\\
\makecell{1-bit $L_2$ norm QSGD} & 54.58$\pm$0.35\%& N.A./N.A. & N.A./N.A.\\
\makecell{1-bit $L_{\infty}$ norm QSGD} & 74.52$\pm$0.58\%& 750/2,950 & $1.64\times10^{8}$/$1.05\times10^{9}$\\
\makecell{TernGrad} & 74.92$\pm$0.42\%& 800/2,800 & $9.61\times10^{7}$/$5.38\times10^{8}$\\
\makecell{{\scriptsize SPARSIGN}SGD ($B=1$)} & 62.34$\pm$0.58\%& 1,550/N.A. & $1.44\times10^{8}$/N.A.\\
\makecell{{\scriptsize EF-SPARSIGN}SGD \\($B_{l}^{(t)}=10$, $B_{g}^{(t)}=1$, $\tau = 1$)} & 78.51$\pm$0.51\%& 300/1,025 & $7.42\times10^{7}$/$4.24\times10^{8}$\\
\bottomrule
\end{tabular}
\end{sc}
\end{center}
\end{table*}

\begin{table*}
\vspace{-0.2in}
\caption{Learning Performance on CIFAR-10 ($\alpha = 0.5$)}
\vspace{-0.1in}
\label{table_cifar_local}
\begin{center}
\begin{sc}
\begin{tabular}{cccccc}
\toprule
Algorithm & \makecell{Final Accuracy} & \makecell{Communications Rounds \\to Achieve 74\%} & \makecell{Communication Overhead \\to Achieve 74\% (bits)}\\
\midrule
\makecell{FedCom-Local5} & 76.03$\pm$0.53\%& 1,025 & $2.75\times10^{9}$\\
\makecell{FedCom-Local10} & 76.20$\pm$0.05\%& 575 & $1.51\times10^{9}$\\
\makecell{FedCom-Local20} & 77.10$\pm$0.29\%& 425 & $1.10\times10^{9}$\\
\makecell{{\scriptsize EF-SPARSIGN}SGD-Local5} & 79.84$\pm$0.17\%& 550 & $3.39\times10^{8}$\\
\makecell{{\scriptsize EF-SPARSIGN}SGD-Local10} & 79.61$\pm$0.25\%& 450 & $2.58\times10^{8}$\\
\makecell{{\scriptsize EF-SPARSIGN}SGD-Local20} & 79.46$\pm$0.09\%& 475 & $2.14\times10^{8}$\\
\bottomrule
\end{tabular}
\end{sc}
\end{center}
\end{table*}

We note that Scaled {\scriptsize SIGN}SGD and Noisy {\scriptsize SIGN}SGD use 1 bit to represent each coordinate of the gradients. 1-bit $L_2$ norm QSGD, 1-bit $L_{\infty}$ norm QSGD and TernGrad are ternary-based methods, which can be understood as special cases of  {\scriptsize SPARSIGN}SGD. For all the ternary-based gradient descent methods, instead of using $\log_2(3)$ bits to represent each coordinate as in \cite{wen2017terngrad}, the lossless Golomb code \cite{golomb1966run} can be adopted to encode the indices of the non-zeros elements. The average number of bits for each index is given by \cite{sattler2019robust}
\begin{equation}
\bar{b} = b^{*} + \frac{1}{1-(1-p)^{2^{b^{*}}}},
\end{equation}
in which $b^{*} = 1 + \lfloor\log_{2}(\frac{\log(\sqrt{5}+\frac{1}{2})}{\log(1-p)})\rfloor$ and $p$ is the sparsity ratio.

{\textbf{Effectiveness of the compressor:} Table \ref{table_fashionmnist} and Table \ref{table_cifar}} compare the performance of the proposed methods with the baselines. We run the algorithms for 200 and 3,000 communication rounds for Fashion-MNIST and CIFAR-10, respectively. We consider full worker participation (i.e., all the 100 workers participate in the training process throughout the training period) for Fashion-MNIST and 20\% worker participating ratio (i.e., 20 workers are randomly sampled to participate during each communication round) for CIFAR-10. The final test accuracy, the number of communication rounds, and the corresponding communication overhead (from the workers to the server) to achieve a test accuracy of 55\%/74\% are presented. ``N.A.” means that the corresponding algorithm does not achieve the required test accuracy. It can be observed that {\scriptsize EF-SPARSIGN}SGD outperforms all the baselines in both the number of communication rounds and the communication overhead. For instance, TernGrad requires $2.7\times$ communication rounds and 27\% more communication overhead to achieve a test accuracy of 74\% on CIFAR-10, and the gap is larger for the other baselines.

\begin{figure}
\vspace{-0.1in}
\centering
\begin{subfigure}
  \centering
  \includegraphics[width=0.23\textwidth]{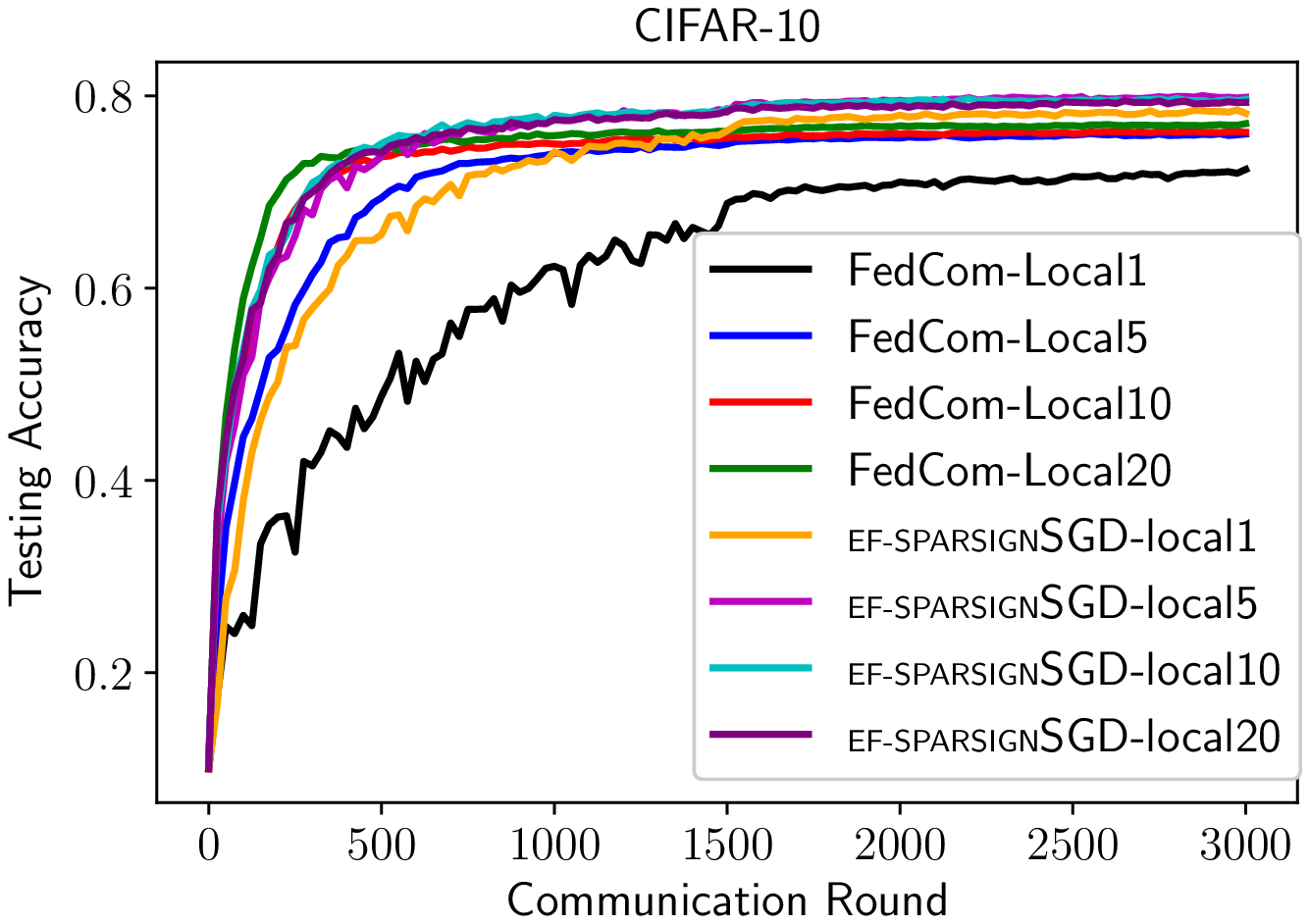}
\end{subfigure}%
\begin{subfigure}
  \centering
  \includegraphics[width=0.23\textwidth]{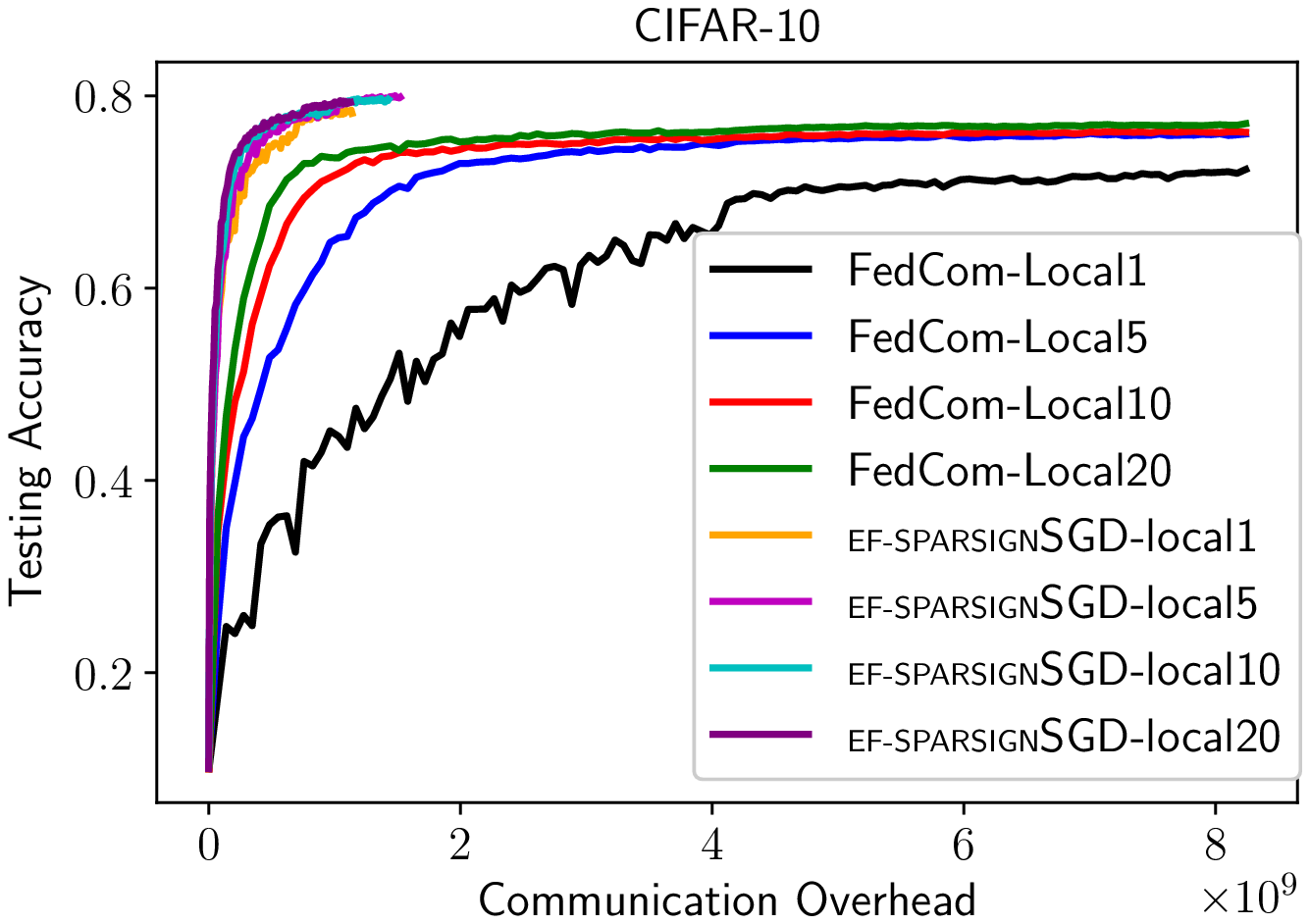}
\end{subfigure}
\vspace{-0.2in}
\caption{The left and right figures compare the testing accuracy of {\scriptsize EF-SPARSIGN}SGD and FedCom with respect to the number of communication rounds and the communication overhead, respectively.}
\label{figure_local}
\vspace{-0.2in}
\end{figure}

{\textbf{Impact of the local training steps:}} Fig.\ref{figure_local} and Table. \ref{table_cifar_local} compare the performance of {\scriptsize EF-SPARSIGN}SGD and FedCom \cite{haddadpour2021federated} that incorporates the compressor of QSGD \cite{alistarh2017qsgd} into FedAvg \cite{mcmahan2017communication}. We consider 8-bit QSGD as suggested in the experiments in \cite{haddadpour2021federated} and obtain the communication overhead (in bits) for each communication round by following Theorem 3.4 in \cite{alistarh2017qsgd}. It can be observed that the communication efficiency of {\scriptsize EF-SPARSIGN}SGD improves as the number of local steps increases from 1 to 20. In addition, {\scriptsize EF-SPARSIGN}SGD outperforms FedCom in both final accuracy and the overall communication overhead to achieve a test accuracy of 74\%, which demonstrates its effectiveness. Similar results are obtained on the CIFAR-100 dataset {\color{black}with $\alpha \in \{0.1,0.3,0.6,1.0\}$}, which can be found in Section \ref{AdditionalResults} of the supplementary material.

\section{Conclusion}
In this work, we addressed the non-convergence issue of {\scriptsize SIGN}SGD in the presence of data heterogeneity. The convergence analysis for {\color{black}sign-based gradient descent methods} was extended to the ternary case, which essentially captured the sparsification and the worker sampling mechanisms in federated learning. A sufficient condition for the convergence of ternary-based gradient descent methods was derived, based on which a magnitude-driven gradient sparsification scheme was incorporated into the sign-based compressor to deal with heterogeneous data distributions. Based on the proposed compressor, two communication-efficient algorithms, {\scriptsize SPARSIGN}SGD and {\scriptsize EF-SPARSIGN}SGD are proposed, and the corresponding convergence is established. Extensive experiments are conducted to validate the effectiveness of the proposed algorithms. The proposed methods are expected to find wide applications in the paradigms like federated learning.

\newpage
\bibliography{Ref-Richeng}
\bibliographystyle{icml2021}

\newpage
\appendix
\onecolumn
\setlength{\abovedisplayskip}{2pt}
\setlength{\belowdisplayskip}{2pt}
\section{Proofs}\label{Proofs}
\setcounter{theorem}{0}
\setcounter{Corollary}{0}
\subsection{Proof of Theorem \ref{SupT1}}
\begin{theorem}\label{SupT1}
Let $u_{1},u_{2},\cdots,u_{M}$ be $M$ known and fixed real numbers with $\frac{1}{M}\sum_{m=1}^{M}u_{m} \neq 0$, and consider random variables $\hat{u}_{m} \in \{-1,0,1\}$ {\color{black}(which is the compressed version of $u_{m}$)}, $1\leq m \leq M$. Denote $P(\hat{u}_{m} = -sign(\frac{1}{M}\sum_{m=1}^{M}u_{m}))=p_{m}$, $P(\hat{u}_{m} = sign(\frac{1}{M}\sum_{m=1}^{M}u_{m}))=q_{m}$ and $P(\hat{u}_{m} = 0)=1-p_{m}-q_{m}$. Let $\Bar{p} = \frac{1}{M}\sum_{m=1}^{M}p_{m}$ and $\Bar{q} = \frac{1}{M}\sum_{m=1}^{M}q_{m}$. If $\Bar{q} > \Bar{p}$, {\color{black}the probability of wrong aggregation is given by}
\begin{equation}\label{SUPPProbabilityOfError}
\begin{split}
P\bigg(sign\bigg(\frac{1}{M}\sum_{m=1}^{M}\hat{u}_{m}\bigg)&\neq sign\bigg(\frac{1}{M}\sum_{m=1}^{M}u_{m}\bigg)\bigg) \leq [1-(\sqrt{\Bar{q}}-\sqrt{\Bar{p}})^2]^{M}.
\end{split}
\end{equation}
\end{theorem}
\begin{proof}
Define a series of random variables $\{X_m\}_{m=1}^{M}$ given by
\begin{equation}
X_{m} =
\begin{cases}
\hfill 1, \hfill &\text{if $\hat{u}_{m} \neq sign\bigg(\frac{1}{M}\sum_{m=1}^{M}u_{m}\bigg)$},\\
\hfill 0, \hfill &\text{if $\hat{u}_{m} = 0$},\\
\hfill -1, \hfill &\text{if $\hat{u}_{m} = sign\bigg(\frac{1}{M}\sum_{m=1}^{M}u_{m}\bigg)$.}
\end{cases}
\end{equation}

Then, we have
\begin{equation}\label{eq1}
\begin{split}
&P\bigg(sign\bigg(\sum_{m=1}^{M}\hat{u}_{m}\bigg)\neq sign\bigg(\sum_{m=1}^{M}u_{m}\bigg)\bigg) = P\left(\sum_{m=1}^{M}X_{m} \geq 0\right).
\end{split}
\end{equation}

For any variable $a > 0$, we have
\begin{equation}
\begin{split}
P\left(\sum_{m=1}^{M}X_m \geq 0\right) &= P\left(e^{a\sum_{m=1}^{M}X_m} \geq e^{0}\right) \leq \frac{\mathbb{E}[e^{a\sum_{m=1}^{M}X_m}]}{e^{0}} = \mathbb{E}[e^{a\sum_{m=1}^{M}X_m}],
\end{split}
\end{equation}
which is due to Markov's inequality, given the fact that $e^{a\sum_{m=1}^{M}X_m}$ is non-negative. For the ease of presentation, let $P(X_{m} = 1) = p_{m}$ and $P(X_{m} = -1) = q_{m}$, we have,
\begin{equation}
\begin{split}
 \mathbb{E}[e^{a\sum_{m=1}^{M}X_m}] &= e^{\ln(\mathbb{E}[e^{a\sum_{m=1}^{M}X_m}])}= e^{\ln(\prod_{m=1}^{M}\mathbb{E}[e^{aX_m}])} = e^{\sum_{m=1}^{M}\ln(\mathbb{E}[e^{aX_m}])} \\
 &=e^{\sum_{m=1}^{M}\ln(e^{a}p_{m}+e^{-a}q_{m}+(1-p_{m}-q_{m}))} \\
 &=e^{M\left(\frac{1}{M}\sum_{m=1}^{M}\ln(e^{a}p_{m}+e^{-a}q_{m}+(1-p_{m}-q_{m}))\right)}\\
 &\leq e^{M\ln(e^{a}\Bar{p}+e^{-a}\Bar{q} + (1-\Bar{p}-\Bar{q}))},
\end{split}
\end{equation}
where $\Bar{p} = \frac{1}{M}\sum_{m=1}^{M}p_{m}$, $\Bar{q} = \frac{1}{M}\sum_{m=1}^{M}q_{m}$ and the inequality is due to Jensen's inequality. Suppose that $\Bar{q} > \Bar{p}$ and let  $a=\ln\left(\sqrt{\frac{\Bar{q}}{\Bar{p}}}\right) > 0$, then
\begin{equation}
\begin{split}
e^{M\ln(e^{a}\Bar{p}+e^{-a}\Bar{q} + (1-\Bar{p}-\Bar{q}))} &= [2\sqrt{\Bar{p}\Bar{q}}+(1-\Bar{p}-\Bar{q})]^{M} = [1-(\sqrt{\Bar{q}}-\sqrt{\Bar{p}})^2]^{M},
\end{split}
\end{equation}
which completes the proof.
\end{proof}

\subsection{Proof of Corollary \ref{SupC1}}
\begin{Corollary}\label{SupC1}
Given the same $u_{1},u_{2},\cdots,u_{M}$ as in Theorem \ref{Theorem1} and consider random variables $\hat{u}_{m}=sparsign(u_{m},B_{m})$, $1\leq m \leq M$. Let $\mathcal{A}$ and $\mathcal{A}^{c}$ denote the sets of workers such that $u_{m} \neq sign(\frac{1}{M}\sum_{m=1}^{M}u_{m}), \forall m \in \mathcal{A}$ and $u_{m} = sign(\frac{1}{M}\sum_{m=1}^{M}u_{m}), \forall m \in \mathcal{A}^{c}$, respectively. Let $S$ denote the set of selected workers, then $\Bar{p} = \frac{1}{M}\sum_{m \in \mathcal{A}}|u_{m}B_{m}P(m \in S)|$ and $\Bar{q} = \frac{1}{M}\sum_{m \in \mathcal{A}^{c}}|u_{m}B_{m}P(m \in S)|$ in (\ref{ProbabilityOfError}).
\end{Corollary}
\begin{proof}
For $Q(\cdot)$ given by (\ref{DefinitionofSparsification3}), we have
\begin{equation}
p_{m} =
\begin{cases}
\hfill 0, \hfill &\text{if $u_{m}\leq0$},\\
\hfill u_{m}B_{m}\Pr(m \in S), \hfill &\text{if $u_{m} > 0$.}
\end{cases}
\end{equation}
\begin{equation}
q_{m} =
\begin{cases}
\hfill -u_{m}B_{m}\Pr(m \in S), \hfill &\text{if $u_{m}<0$},\\
\hfill 0, \hfill &\text{if $u_{m} \geq 0$.}
\end{cases}
\end{equation}
With some algebra, we can show that $\Bar{p} = \frac{1}{M}\sum_{m \in \mathcal{A}}|u_{m}B_{m}\Pr(m \in S)|$ and $\Bar{q} = \frac{1}{M}\sum_{m \in \mathcal{A}^{c}}|u_{m}B_{m}\Pr(m \in S)|$.
\end{proof}

\subsection{Proof of Theorem \ref{SPconvergerate}}
\begin{theorem}\label{SPconvergerate}
Suppose Assumptions \ref{A1}, \ref{A2} and \ref{A4} are satisfied, and the learning rate is set as $\eta=\frac{1}{\sqrt{Td}}$. Then by running Algorithm \ref{QuantizedSIGNSGD} (termed {\scriptsize SPARSIGN}SGD) with $Q(\boldsymbol{g}_{m}^{(t)},\boldsymbol{B}^{(t)}_{m}) = sparsign(\boldsymbol{g}_{m}^{(t)},\boldsymbol{B}^{(t)}_{m}))$, where $\boldsymbol{B}^{(t)}_{m,i} = B_{i}^{(t)}$, $\forall m$, and $\mathcal{C}(\cdot) = sign(\cdot)$ for $T$ iterations, we have
\begin{equation}
\color{black}
\begin{split}
\frac{1}{T}\sum_{t=1}^{T}\sum_{i=1}^{d}(1-2\rho_{i})(1-2\kappa_{i}^{(t)})|\nabla F(w^{(t)})_{i}| &\leq
\frac{\mathbb{E}[F(w^{(0)}) - F(w^{(T+1)})]\sqrt{d}}{\sqrt{T}} + \frac{L\sqrt{d}}{2\sqrt{T}}\\
&\leq\frac{(F(w^{(0)}) - F^{*})\sqrt{d}}{\sqrt{T}} + \frac{L\sqrt{d}}{2\sqrt{T}},
\end{split}
\end{equation}
where $\kappa_{m,i}^{(t)} = \mathbb{E}\left[\left[1-B_{i}^{(t)}p_{s}\left(\frac{|\frac{1}{M}\sum_{m=1}^{M}\boldsymbol{g}_{m,i}^{(t)}|}{\sqrt{\frac{1}{M}\sum_{m \in \mathcal{A}^{c}_{(t)}}|\boldsymbol{g}_{m,i}^{(t)}|}+\sqrt{\frac{1}{M}\sum_{m \in \mathcal{A}_{(t)}}|\boldsymbol{g}_{m,i}^{(t)}|}}\right)^2\right]^{M}\right]$, $\mathcal{A}_{(t)}$ and $\mathcal{A}^{c}_{(t)}$ are the set of workers such that $sign\left(\boldsymbol{g}_{m,i}^{(t)}\right) \neq sign\left(\frac{1}{M}\sum_{m=1}^{M}\boldsymbol{g}_{m,i}^{(t)}\right)$ and $sign\left(\boldsymbol{g}_{m,i}^{(t)}\right) = sign\left(\frac{1}{M}\sum_{m=1}^{M}\boldsymbol{g}_{m,i}^{(t)}\right)$, respectively, and the expectation is over the randomness of the gradients $\boldsymbol{g}_{m}^{(t)}$.
\end{theorem}

The proof of Theorem \ref{SPconvergerate} follows the well-known strategy of relating the norm of the gradient to the expected improvement of the global objective in a single iteration. Then accumulating the improvement over the iterations yields the convergence rate of the algorithm.

\begin{proof}
According to Assumption \ref{A2}, we have
\begin{equation}
\begin{split}
F(w^{(t+1)}) - F(w^{(t)}) &\leq \langle\nabla F(w^{(t)}), w^{(t+1)}-w^{(t)}\rangle + \frac{L}{2}||w^{(t+1)}-w^{(t)}||^2 \\
& =-\eta \langle\nabla F(w^{(t)}), sign\bigg(\frac{1}{M}\sum_{m=1}^{M}Q(\boldsymbol{g}_{m}^{(t)},\boldsymbol{B}^{(t)}_{m})\bigg)\rangle + \frac{L}{2}\bigg|\bigg|\eta sign\bigg(\frac{1}{M}\sum_{m=1}^{M}Q(\boldsymbol{g}_{m}^{(t)},\boldsymbol{B}^{(t)}_{m})\bigg)\bigg|\bigg|^2 \\
& = -\eta \langle\nabla F(w^{(t)}), sign\bigg(\frac{1}{M}\sum_{m=1}^{M}Q(\boldsymbol{g}_{m}^{(t)},\boldsymbol{B}^{(t)}_{m})\bigg)\rangle + \frac{L\eta^2d}{2} \\
& = -\eta ||\nabla F(w^{(t)})||_{1} + \frac{L\eta^2d}{2} + 2\eta\sum_{i=1}^{d}|\nabla F(w^{(t)})_{i}|\mathds{1}_{sign(\frac{1}{M}\sum_{m=1}^{M}Q(\boldsymbol{g}_{m}^{(t)},\boldsymbol{B}^{(t)}_{m})_{i})\neq sign(\nabla F(w^{(t)})_{i})},
\end{split}
\end{equation}
where $\nabla F(w^{(t)})_{i}$ is the $i$-th entry of the vector $\nabla F(w^{(t)})$ and $\eta$ is the learning rate. Taking expectation on both sides yields

\begin{equation}\label{convergencee1}
\color{black}
\begin{split}
&\mathbb{E}[F(w^{(t+1)}) - F(w^{(t)})] \leq -\eta ||\nabla F(w^{(t)})||_{1} + \frac{L\eta^2d}{2} \\
&+2\eta\mathbb{E}\bigg[\sum_{i=1}^{d}|\nabla F(w^{(t)})_{i}|P\bigg(sign\bigg(\frac{1}{M}\sum_{m=1}^{M}Q(\boldsymbol{g}_{m}^{(t)},\boldsymbol{B}^{(t)}_{m})_{i}\bigg)\neq sign(\nabla F(w^{(t)})_{i})\bigg)\bigg]\\
\end{split}
\end{equation}

Denote $\rho_{i} = P\left(sign\left(\frac{1}{M}\sum_{m=1}^{M}\boldsymbol{g}_{m,i}^{(t)}\right) \neq sign(\nabla F(w^{(t)})_{i})\right)$ and \\
$\phi_{i} = P\left(sign\left(\frac{1}{M}\sum_{m=1}^{M}Q(\boldsymbol{g}_{m}^{(t)},\boldsymbol{B}^{(t)}_{m})_{i}\right) \neq sign\left(\frac{1}{M}\sum_{m=1}^{M}\boldsymbol{g}_{m,i}^{(t)}\right)\right)$, we have
\begin{equation}\label{probstochastic}
\begin{split}
&\mathbb{E}\bigg[P\left(sign\left(\frac{1}{M}\sum_{m=1}^{M}Q(\boldsymbol{g}_{m}^{(t)},\boldsymbol{B}^{(t)}_{m})_{i}\right)\neq sign\left(\nabla F(w^{(t)})_{i}\right)\right)\bigg] = \mathbb{E}[\phi_{i}(1-\rho_{i})+(1-\phi_{i})\rho_{i}] = \mathbb{E}[(1-2\rho_{i})\phi_{i} + \rho_{i}] \\
&\leq \rho_{i} + (1-2\rho_{i})\mathbb{E}\left[\left[1-\left(\sqrt{\frac{1}{M}\sum_{m \in \mathcal{A}^{c}_{(t)}}|\boldsymbol{g}_{m,i}^{(t)}B_{i}^{(t)}p_{s}|}-\sqrt{\frac{1}{M}\sum_{m \in \mathcal{A}_{(t)}}|\boldsymbol{g}_{m,i}^{(t)}B_{i}^{(t)}p_{s}|}\right)^2\right]^{M}\right]\\
&= \rho_{i} + (1-2\rho_{i})\mathbb{E}\left[\left[1-B_{i}^{(t)}p_{s}\left(\sqrt{\frac{1}{M}\sum_{m \in \mathcal{A}^{c}_{(t)}}|\boldsymbol{g}_{m,i}^{(t)}|}-\sqrt{\frac{1}{M}\sum_{m \in \mathcal{A}_{(t)}}|\boldsymbol{g}_{m,i}^{(t)}|}\right)^2\right]^{M}\right]\\
&= \rho_{i} + (1-2\rho_{i})\mathbb{E}\left[\left[1-B_{i}^{(t)}p_{s}\left(\frac{\frac{1}{M}\sum_{m \in \mathcal{A}^{c}_{(t)}}|\boldsymbol{g}_{m,i}^{(t)}|-\frac{1}{M}\sum_{m \in \mathcal{A}_{(t)}}|\boldsymbol{g}_{m,i}^{(t)}|}{\sqrt{\frac{1}{M}\sum_{m \in \mathcal{A}^{c}_{(t)}}|\boldsymbol{g}_{m,i}^{(t)}|}+\sqrt{\frac{1}{M}\sum_{m \in \mathcal{A}_{(t)}}|\boldsymbol{g}_{m,i}^{(t)}|}}\right)^2\right]^{M}\right] \\
&= \rho_{i} + (1-2\rho_{i})\mathbb{E}\left[\left[1-B_{i}^{(t)}p_{s}\left(\frac{|\frac{1}{M}\sum_{m=1}^{M}\boldsymbol{g}_{m,i}^{(t)}|}{\sqrt{\frac{1}{M}\sum_{m \in \mathcal{A}^{c}_{(t)}}|\boldsymbol{g}_{m,i}^{(t)}|}+\sqrt{\frac{1}{M}\sum_{m \in \mathcal{A}_{(t)}}|\boldsymbol{g}_{m,i}^{(t)}|}}\right)^2\right]^{M}\right] \\
&\triangleq  \rho_{i} + (1-2\rho_{i})\kappa_{i}^{(t)}.
\end{split}
\end{equation}
where $\mathcal{A}_{(t)}$ and $\mathcal{A}^{c}_{(t)}$ are the set of workers such that $sign\left(\boldsymbol{g}_{m,i}^{(t)}\right) \neq sign\left(\frac{1}{M}\sum_{m=1}^{M}\boldsymbol{g}_{m,i}^{(t)}\right)$ and $sign\left(\boldsymbol{g}_{m,i}^{(t)}\right) = sign\left(\frac{1}{M}\sum_{m=1}^{M}\boldsymbol{g}_{m,i}^{(t)}\right)$, respectively.

Plugging (\ref{probstochastic}) into (\ref{convergencee1}), we can obtain
\begin{equation}
\color{black}
\begin{split}
\mathbb{E}[F(w^{(t+1)}) - F(w^{(t)})] &\leq -\eta ||\nabla F(w^{(t)})||_{1} + \frac{L\eta^2d}{2}  + 2\eta\sum_{i=1}^{d}|\nabla F(w^{(t)})_{i}|[ \rho_{i} + (1-2\rho_{i})\kappa_{i}^{(t)}].\\
&\leq -\eta \sum_{i=1}^{d}(1-2\rho_{i})(1-2\kappa_{i}^{(t)})|\nabla F(w^{(t)})_{i}| + \frac{L\eta^2d}{2}
\end{split}
\end{equation}
Adjusting the above inequality and averaging both sides over $t=1,2,\cdots,T$, we can obtain
\begin{equation}
\color{black}
\begin{split}
\frac{1}{T}\sum_{t=1}^{T}\sum_{i=1}^{d}\eta(1-2\rho_{i})(1-2\kappa_{i}^{(t)})|\nabla F(w^{(t)})_{i}| &\leq \frac{\mathbb{E}[F(w^{(0)}) - F(w^{(T+1)})]}{T} + \frac{L\eta^2d}{2}\\
\end{split}
\end{equation}
Letting $\eta=\frac{1}{\sqrt{dT}}$ and dividing both sides by $\eta$ gives
\begin{equation}
\color{black}
\begin{split}
\frac{1}{T}\sum_{t=1}^{T}\sum_{i=1}^{d}(1-2\rho_{i})(1-2\kappa_{i}^{(t)})|\nabla F(w^{(t)})_{i}| &\leq
\frac{\mathbb{E}[F(w^{(0)}) - F(w^{(T+1)})]\sqrt{d}}{\sqrt{T}} + \frac{L\sqrt{d}}{2\sqrt{T}}\\
&\leq\frac{(F(w^{(0)}) - F^{*})\sqrt{d}}{\sqrt{T}} + \frac{L\sqrt{d}}{2\sqrt{T}},
\end{split}
\end{equation}
which completes the proof.
\end{proof}

\subsection{Proof of Theorem \ref{SuppEFDPSIGNConvergence2}}
\begin{theorem}\label{SuppEFDPSIGNConvergence2}
When Assumptions \ref{A1}, \ref{A2} and \ref{A3} are satisfied, by running Algorithm \ref{QuantizedSIGNSGDLocal} with $\eta_{L} = \frac{1}{\sqrt{Td}\tau}$, $\eta = \tau$, $\boldsymbol{B}_{m}^{(t,c)} = B_{l}^{(t)}\cdot\boldsymbol{1}$, $\boldsymbol{B}^{(t)}_{m} = B_{g}^{(t)}\cdot\boldsymbol{1} \forall m$, $B_{l}^{(t)}B_{g}^{(t)} = B$, we have
\begin{equation}
\begin{split}
\frac{1}{T}\sum_{t=0}^{T-1}||\nabla F(w^{(t)})||^2 \leq \frac{(F(w^{(0)})-F^{*})\sqrt{d}}{B\tau\sqrt{T}} + \frac{(1+L+L^2\beta)\sqrt{d}}{B\tau\sqrt{T}}+ \frac{L^2(\tau+1)(2\tau+1)}{6T\tau^2}.
\end{split}
\end{equation}
\end{theorem}

Before proving Theorem \ref{SuppEFDPSIGNConvergence2}, we first show the following lemmas.
\begin{Lemma}\label{Recurrence2}
Let $y^{(t)} = w^{(t)} - \eta\eta_{L} \tilde{\boldsymbol{e}}^{(t)}$, we have
\begin{equation}
y^{(t+1)} = y^{(t)} - \eta\eta_{L}\frac{1}{|S^{(t)}|}\sum_{m \in S^{(t)}}\Delta_{m}^{(t)}.
\end{equation}
\end{Lemma}

\begin{proof}
\begin{equation}
\begin{split}
   y^{(t+1)} &=  w^{(t+1)} - \eta\eta_{L} \tilde{\boldsymbol{e}}^{(t+1)}\\
     &= w^{(t)} - \eta\eta_{L}\tilde{\boldsymbol{g}}^{(t)} - \eta\eta_{L} \tilde{\boldsymbol{e}}^{(t+1)} \\
     &= w^{(t)} - \eta\eta_{L}\bigg(\frac{1}{|S^{(t)}|}\sum_{m \in S^{(t)}}\Delta_{m}^{(t)} + \tilde{\boldsymbol{e}}^{(t)} - \tilde{\boldsymbol{e}}^{(t+1)}\bigg) - \eta\eta_{L} \tilde{\boldsymbol{e}}^{(t+1)} \\
     &= w^{(t)} - \eta\eta_{L}\frac{1}{|S^{(t)}|}\sum_{m \in S^{(t)}}\Delta_{m}^{(t)} - \eta\eta_{L} \tilde{\boldsymbol{e}}^{(t)} \\
     & = y^{(t)} - \eta\eta_{L}\frac{1}{|S^{(t)}|}\sum_{m \in S^{(t)}}\Delta_{m}^{(t)}.
\end{split}
\end{equation}
\end{proof}

\begin{Lemma}\label{bound2}
There exists a positive constant $\beta > 0$ such that
$\mathbb{E}[||\tilde{\boldsymbol{e}}^{(t)}||^2_{2}] \leq \beta d, \forall t$.
\end{Lemma}

\begin{proof}
Since $\mathcal{C}(\cdot)$ is an $\alpha$-approximate compressor, it can be shown that
\begin{equation}
\begin{split}
\mathbb{E}||\tilde{\boldsymbol{e}}^{(t+1)}||_{2}^{2} &\leq (1-\alpha)\bigg|\bigg|\frac{1}{|S^{(t)}|}\sum_{m \in S^{(t)}}\Delta_{m}^{(t)}+\tilde{\boldsymbol{e}}^{(t)}\bigg|\bigg|_{2}^{2} \\
&\leq (1-\alpha)(1+\rho)\mathbb{E}||\tilde{\boldsymbol{e}}^{(t)}||_2^2 + (1-\alpha)\bigg(1+\frac{1}{\rho}\bigg)\mathbb{E}\bigg|\bigg|\frac{1}{|S^{(t)}|}\sum_{m \in S^{(t)}}\Delta_{m}^{(t)}\bigg|\bigg|_2^2 \\
&\leq \sum_{j=0}^{t}[(1-\alpha)(1+\rho)]^{t-j}(1-\alpha)\bigg(1+\frac{1}{\rho}\bigg)\mathbb{E}\bigg|\bigg|\frac{1}{|S^{(t)}|}\sum_{m \in S^{(j)}}\Delta_{m}^{(j)}\bigg|\bigg|_2^2 \\
&\leq \frac{(1-\alpha)\bigg(1+\frac{1}{\rho}\bigg)}{1-(1-\alpha)(1+\rho)}d,
\end{split}
\end{equation}
where we invoke Young's inequality recurrently and $\rho$ can be any positive constant. Therefore, there exists some constant $\beta > 0$ such that $\mathbb{E}[||\tilde{\boldsymbol{e}}^{(t)}||^2_{2}] \leq \beta d, \forall t$.
\end{proof}

Now, we are ready to prove Theorem \ref{SuppEFDPSIGNConvergence2}.
\begin{proof}
Let $y^{(t)} = w^{(t)} - \eta\eta_{L} \tilde{\boldsymbol{e}}^{(t)}$, according to Lemma \ref{Recurrence2}, we have
\begin{equation}\label{Convergence3}
\begin{split}
\mathbb{E}[F(y^{(t+1)}) - F(y^{(t)})] &\leq -\eta\eta_{L}\mathbb{E}\bigg[\langle\nabla F(y^{(t)}), \frac{1}{|S^{(t)}|}\sum_{m \in S^{(t)}}\Delta_{m}^{(t)}\rangle\bigg] + \frac{L}{2}\mathbb{E}\bigg[\bigg|\bigg|\eta\eta_{L}\frac{1}{|S^{(t)}|}\sum_{m \in S^{(t)}}\Delta_{m}^{(t)}\bigg|\bigg|^2_{2}\bigg] \\
&= \eta\eta_{L}\mathbb{E}\bigg[\langle\nabla F(w^{(t)}) - \nabla F(y^{(t)}), \frac{1}{|S^{(t)}|}\sum_{m \in S^{(t)}}\Delta_{m}^{(t)}\rangle\bigg] + \frac{L\eta^2\eta_{L}^2}{2}\mathbb{E}\bigg[\bigg|\bigg|\frac{1}{|S^{(t)}|}\sum_{m \in S^{(t)}}\Delta_{m}^{(t)}\bigg|\bigg|^2_{2}\bigg] \\
& -  \eta\eta_{L}\mathbb{E}\bigg[\langle\nabla F(w^{(t)}), \frac{1}{|S^{(t)}|}\sum_{m \in S^{(t)}}\Delta_{m}^{(t)}\rangle\bigg].
\end{split}
\end{equation}

The second term can be bounded as follows.
\begin{equation}\label{Convergence_1}
\begin{split}
\frac{L\eta^2\eta_{L}^{2}}{2}\mathbb{E}\bigg[\bigg|\bigg|\frac{1}{|S^{(t)}|}\sum_{m \in S^{(t)}}\Delta_{m}^{(t)}\bigg|\bigg|^2_{2}\bigg] \leq \frac{L\eta^2\eta_{L}^{2}d}{2}.
\end{split}
\end{equation}
We then bound the first term, in particular, we have
\begin{equation}\label{Convergence4}
\begin{split}
\langle\nabla F(w^{(t)})- \nabla F(y^{(t)}), \frac{1}{|S^{(t)}|}\sum_{m \in S^{(t)}}\Delta_{m}^{(t)}\rangle & \leq \frac{\eta\eta_{L}}{2}||\frac{1}{|S^{(t)}|}\sum_{m \in S^{(t)}}\Delta_{m}^{(t)}||^2_{2} + \frac{1}{2\eta\eta_{L}}||\nabla F(w^{(t)})-\nabla F(y^{(t)})||^2_{2} \\
&\leq \frac{\eta\eta_{L}}{2}||\frac{1}{|S^{(t)}|}\sum_{m \in S^{(t)}}\Delta_{m}^{(t)}||^2_{2} + \frac{L^2}{2\eta\eta_{L}}||y^{(t)} - w^{(t)}||^2_{2} \\
&= \frac{\eta\eta_{L}}{2}||\frac{1}{|S^{(t)}|}\sum_{m \in S^{(t)}}\Delta_{m}^{(t)}||^2_{2} + \frac{L^2\eta\eta_{L}}{2}||\tilde{\boldsymbol{e}}^{(t)}||^2_{2} \\
&\leq \frac{\eta\eta_{L}}{2}||\frac{1}{|S^{(t)}|}\sum_{m \in S^{(t)}}\Delta_{m}^{(t)}||^2_{2} + \frac{L^2\eta\eta_{L}}{2}\beta d\\
&\leq \frac{\eta\eta_{L} d}{2} + \frac{L^2\eta\eta_{L} \beta d}{2}
\end{split}
\end{equation}
where the second inequality is due to the $L$-smoothness of $F$.

When $\boldsymbol{B}_{m}^{(t,c)} = B_{l}^{(t)}\cdot\boldsymbol{1}$, $\boldsymbol{B}^{(t)}_{m} = B_{g}^{(t)}\cdot\boldsymbol{1} \forall m$, we can bound the last term as follows.
\begin{equation}\label{Convergence5}
\begin{split}
&-\mathbb{E}\bigg[\langle\nabla F(w^{(t)}), \frac{1}{|S^{(t)}|}\sum_{m \in S^{(t)}}\Delta_{m}^{(t)}\rangle\bigg] \\
&=-\mathbb{E}\bigg[\langle\nabla F(w^{(t)}), \frac{1}{|S^{(t)}|}\sum_{m \in S^{(t)}}Q(\sum_{c=0}^{\tau-1}Q(\boldsymbol{g}_{m}^{(t,c)},\boldsymbol{B}_{m}^{(t,c)}),\boldsymbol{B}^{(t)}_{m})\rangle\bigg] \\
& = -\mathbb{E}\bigg[\langle\nabla F(w^{(t)}), \frac{1}{|S^{(t)}|}\sum_{m \in S^{(t)}}B_{g}^{(t)}\sum_{c=0}^{\tau-1}B_{l}^{(t)}\nabla f_{m}(w^{(t, c)}_{m})\rangle\bigg] \\
& = -\mathbb{E}\bigg[\langle\nabla F(w^{(t)}), \frac{B_{l}^{(t)}B_{g}^{(t)}}{|S^{(t)}|}\sum_{m \in S^{(t)}}\sum_{c=0}^{\tau-1}\nabla f_{m}(w^{(t, c)}_{m})\rangle\bigg] = -\mathbb{E}\bigg[\langle\nabla F(w^{(t)}), B_{l}^{(t)}B_{g}^{(t)}\frac{1}{M}\sum_{m=1}^{M}\sum_{c=0}^{\tau-1}\nabla f_{m}(w^{(t, c)}_{m})\rangle\bigg] \\
& = \mathbb{E}\bigg[\langle\nabla F(w^{(t)}), -B_{l}^{(t)}B_{g}^{(t)}\frac{1}{M}\sum_{m=1}^{M}\sum_{c=0}^{\tau-1}\nabla f_{m}(w^{(t, c)}_{m}) + B_{l}^{(t)}B_{g}^{(t)}\frac{1}{M}\sum_{m=1}^{M}\sum_{c=0}^{\tau-1}\nabla F(w^{(t)}) \\
&- B_{l}^{(t)}B_{g}^{(t)}\frac{1}{M}\sum_{m=1}^{M}\sum_{c=0}^{\tau-1}\nabla F(w^{(t)})\rangle\bigg] \\
& = -B_{l}^{(t)}B_{g}^{(t)}\tau||\nabla F(w^{(t)})||_{2}^{2} + \mathbb{E}\bigg[\langle\nabla F(w^{(t)}), -B_{l}^{(t)}B_{g}^{(t)}\frac{1}{M}\sum_{m=1}^{M}\sum_{c=0}^{\tau-1}\nabla f_{m}(w^{(t, c)}_{m}) + B_{l}^{(t)}B_{g}^{(t)}\frac{1}{M}\sum_{m=1}^{M}\sum_{c=0}^{\tau-1}\nabla F(w^{(t)}) \rangle\bigg].
\end{split}
\end{equation}

\begin{equation}\label{lastterm}
\begin{split}
&\mathbb{E}\bigg[\langle\nabla F(w^{(t)}), -B_{l}^{(t)}B_{g}^{(t)}\frac{1}{M}\sum_{m=1}^{M}\sum_{c=0}^{\tau-1}\nabla f_{m}(w^{(t, c)}_{m}) + B_{l}^{(t)}B_{g}^{(t)}\frac{1}{M}\sum_{m=1}^{M}\sum_{c=0}^{\tau-1}\nabla F(w^{(t)}) \rangle\bigg] \\
& = B_{l}^{(t)}B_{g}^{(t)}\langle\nabla F(w^{(t)}), \mathbb{E}\bigg[\frac{1}{M}\sum_{m=1}^{M}\sum_{c=0}^{\tau-1}-\nabla f_{m}(w^{(t, c)}_{m})+\frac{1}{M}\sum_{m=1}^{M}\sum_{c=0}^{\tau-1}\nabla f_{m}(w^{(t)})\bigg]\rangle\\
&=B_{l}^{(t)}B_{g}^{(t)}\langle\sqrt{\tau}\nabla F(w^{(t)}),\frac{1}{\sqrt{\tau}} \mathbb{E}\bigg[\frac{1}{M}\sum_{m=1}^{M}\sum_{c=0}^{\tau-1}-\nabla f_{m}(w^{(t, c)}_{m})+\frac{1}{M}\sum_{m=1}^{M}\sum_{c=0}^{\tau-1}\nabla f_{m}(w^{(t)})\bigg]\rangle \\
&=\frac{B_{l}^{(t)}B_{g}^{(t)}}{2}\bigg[||\sqrt{\tau}\nabla F(w^{(t)})||^2 + \mathbb{E}\bigg|\bigg|\frac{1}{\sqrt{\tau}} \bigg[\frac{1}{M}\sum_{m=1}^{M}\sum_{c=0}^{\tau-1}-\nabla f_{m}(w^{(t, c)}_{m})+\frac{1}{M}\sum_{m=1}^{M}\sum_{c=0}^{\tau-1}\nabla f_{m}(w^{(t)})\bigg]\bigg|\bigg|^2 \\
&- \mathbb{E}\bigg|\bigg|\frac{1}{\sqrt{\tau}} \bigg[\frac{1}{M}\sum_{m=1}^{M}\sum_{c=0}^{\tau-1}-\nabla f_{m}(w^{(t, c)}_{m})\bigg]\bigg|\bigg|^2\bigg] \\
&\leq \frac{B_{l}^{(t)}B_{g}^{(t)}}{2}\bigg[||\sqrt{\tau}\nabla F(w^{(t)})||^2 + \mathbb{E}\bigg|\bigg|\frac{1}{\sqrt{\tau}} \bigg[\frac{1}{M}\sum_{m=1}^{M}\sum_{c=0}^{\tau-1}-\nabla f_{m}(w^{(t, c)}_{m})+\frac{1}{M}\sum_{m=1}^{M}\sum_{c=0}^{\tau-1}\nabla f_{m}(w^{(t)})\bigg]\bigg|\bigg|\bigg] \\
&\leq \frac{B_{l}^{(t)}B_{g}^{(t)}\tau}{2}||\nabla F(w^{(t)})||^2 + \frac{B_{l}^{(t)}B_{g}^{(t)}}{2}\mathbb{E}\bigg|\bigg|\frac{1}{\sqrt{\tau}} \bigg[\frac{1}{M}\sum_{m=1}^{M}\sum_{c=0}^{\tau-1}(\nabla f_{m}(w^{(t)})-\nabla f_{m}(w^{(t, c)}_{m}))\bigg]\bigg|\bigg|^2\\
&= \frac{B_{l}^{(t)}B_{g}^{(t)}\tau}{2}||\nabla F(w^{(t)})||^2 + \frac{B_{l}^{(t)}B_{g}^{(t)}}{2M^2\tau}\mathbb{E}\bigg|\bigg| \bigg[\sum_{m=1}^{M}\sum_{c=0}^{\tau-1}(\nabla f_{m}(w^{(t)})-\nabla f_{m}(w^{(t, c)}_{m}))\bigg]\bigg|\bigg|^2 \\
&\leq \frac{B_{l}^{(t)}B_{g}^{(t)}\tau}{2}||\nabla F(w^{(t)})||^2 + \frac{B_{l}^{(t)}B_{g}^{(t)}}{2M}\sum_{m=1}^{M}\sum_{c=0}^{\tau-1}\mathbb{E}|| \nabla f_{m}(w^{(t)})-\nabla f_{m}(w^{(t, c)}_{m})||^2 \\
&\leq \frac{B_{l}^{(t)}B_{g}^{(t)}\tau}{2}||\nabla F(w^{(t)})||^2 + \frac{B_{l}^{(t)}B_{g}^{(t)}L^2}{2M }\sum_{m=1}^{M}\sum_{c=0}^{\tau-1}\mathbb{E}||w^{(t)}-w^{(t, c)}_{m}||^2.
\end{split}
\end{equation}

In addition, we have
\begin{equation}\label{localdistance}
\begin{split}
\mathbb{E}||w^{(t)}-w^{(t, c)}_{m}||^2 &= \mathbb{E}\bigg[\bigg|\bigg|w^{(t)}-(w^{(t)} - \sum_{j=0}^{c-1}\eta_{L}Q(\boldsymbol{g}_{m}^{(t,j)},B_{m}^{(t,j)}))\bigg|\bigg|^2\bigg] = \mathbb{E}\bigg[\bigg|\bigg| \sum_{j=0}^{c-1}\eta_{L}Q(\boldsymbol{g}_{m}^{(t,j)},B_{m}^{(t,j)})\bigg|\bigg|^2\bigg] \leq \eta_{L}^2c^2d.
\end{split}
\end{equation}

Plugging (\ref{lastterm}) and (\ref{localdistance}) into (\ref{Convergence5}) yields
\begin{equation}\label{Convergence10}
\begin{split}
&-\mathbb{E}\bigg[\langle\nabla F(w^{(t)}), \frac{1}{|S^{(t)}|}\sum_{m \in S^{(t)}}\Delta_{m}^{(t)}\rangle\bigg] \\
& = -B_{l}^{(t)}B_{g}^{(t)}\tau||\nabla F(w^{(t)})||_{2}^{2} + \mathbb{E}\bigg[\langle\nabla F(w^{(t)}), -B_{l}^{(t)}B_{g}^{(t)}\frac{1}{M}\sum_{m=1}^{M}\sum_{c=0}^{\tau-1}\nabla f_{m}(w^{(t, c)}_{m}) \\
&+ B_{l}^{(t)}B_{g}^{(t)}\frac{1}{M}\sum_{m=1}^{M}\sum_{c=0}^{\tau-1}\nabla F(w^{(t)}) \rangle\bigg]\\
&\leq-B_{l}^{(t)}B_{g}^{(t)}\tau||\nabla F(w^{(t)})||_{2}^{2} + \frac{B_{l}^{(t)}B_{g}^{(t)}\tau}{2}||\nabla F(w^{(t)})||^2 + \frac{B_{l}^{(t)}B_{g}^{(t)}L^2}{2M}\sum_{m=1}^{M}\sum_{c=0}^{\tau-1}\eta_{L}^2c^2d\\
&=-\frac{B_{l}^{(t)}B_{g}^{(t)}\tau}{2}||\nabla F(w^{(t)})||^2 + \frac{B_{l}^{(t)}B_{g}^{(t)}L^2\eta_{L}^2d}{2M}\sum_{m=1}^{M}\sum_{c=0}^{\tau-1}c^2\\
&=-\frac{B_{l}^{(t)}B_{g}^{(t)}\tau}{2}||\nabla F(w^{(t)})||^2 + \frac{B_{l}^{(t)}B_{g}^{(t)}L^2\eta_{L}^2d}{2}\frac{\tau(\tau+1)(2\tau+1)}{6}\\
\end{split}
\end{equation}

Plugging (\ref{Convergence_1}), (\ref{Convergence4}) and (\ref{Convergence10}) into (\ref{Convergence3}) yields

\begin{equation}\label{Convergence6}
\begin{split}
&\mathbb{E}[F(y^{(t+1)}) - F(y^{(t)})] \\
&\leq \frac{\eta^2\eta_{L}^{2} d}{2} + \frac{L^2\eta^2\eta_{L}^{2} \beta d}{2} + \frac{L\eta^2\eta_{L}^{2} d}{2} -\frac{B_{l}^{(t)}B_{g}^{(t)}\tau\eta\eta_{L}}{2}||\nabla F(w^{(t)})||^2 + \frac{B_{l}^{(t)}B_{g}^{(t)}L^2\eta\eta_{L}^3d}{2}\frac{\tau(\tau+1)(2\tau+1)}{6}.
\end{split}
\end{equation}

Rewriting (\ref{Convergence6}) and taking average over $t=0,1,2,\cdots,T-1$ on both sides yields
\begin{equation}
\begin{split}
&\frac{1}{T}\sum_{t=0}^{T-1}||\nabla F(w^{(t)})||^2 \leq \sum_{t=0}^{T-1}\frac{2\mathbb{E}[F(y^{(t)}) - F(y^{(t+1)})]}{B_{l}^{(t)}B_{g}^{(t)}\tau\eta\eta_{L} T} + \frac{(\eta\eta_{L}+L\eta\eta_{L}+L^2\eta\eta_{L}\beta)d}{B_{l}^{(t)}B_{g}^{(t)}\tau} + \frac{L^2\eta_{L}^2d(\tau+1)(2\tau+1)}{6}.\\
\end{split}
\end{equation}

Taking $\eta_{L} = \frac{1}{\sqrt{Td}\tau}$, $\eta = \tau$ and $w^{(0)}=y^{(0)}$ and $B_{l}^{(t)}B_{g}^{(t)} = B, \forall t$ yields
\begin{equation}
\begin{split}
\frac{1}{T}\sum_{t=0}^{T-1}||\nabla F(w^{(t)})||^2 \leq \frac{(F(w^{(0)})-F^{*})\sqrt{d}}{B\tau\sqrt{T}} + \frac{(1+L+L^2\beta)\sqrt{d}}{B\tau\sqrt{T}}+ \frac{L^2(\tau+1)(2\tau+1)}{6T\tau^2}.
\end{split}
\end{equation}
\end{proof}
\begin{Remark}
Despite not requiring the bounded gradient dissimilarity assumption, another assumption that we implicitly make is $|\boldsymbol{g}_{m,i}^{(t)}| \leq \frac{1}{B_{l}^{(t)}}, \forall m,i,t$, as in Definition \ref{definitioncompressor}. This can be easily satisfied given the bounded gradient assumption, which is commonly adopted in the literature (e.g., \cite{tang2019doublesqueeze,karimireddy2019error,chen2019distributed,zheng2019communication}) that aims to address the non-convergence issue of {\scriptsize SIGN}SGD. More specifically, Theorem \ref{EFDPSIGNConvergence2} suggests that when $\lim_{T\rightarrow\infty}\frac{1}{B\sqrt{T}} = 0$, the convergence of Algorithm \ref{QuantizedSIGNSGDLocal} is guaranteed. For any fixed $B_{g}^{(t)}$, if $\lim_{T\rightarrow\infty}\frac{|\boldsymbol{g}_{m,i}^{(t)}|}{\sqrt{T}} = 0$, we can find some $B_{l}^{(t)}$ such that $|\boldsymbol{g}_{m,i}^{(t)}| \leq \frac{1}{B_{l}^{(t)}}$ and $\lim_{T\rightarrow\infty}\frac{1}{B\sqrt{T}} = 0$ are satisfied simultaneously. That being said, as long as $|\boldsymbol{g}_{m,i}^{(t)}| \leq T^{a}, \forall m,i,t$, for some $a < \frac{1}{2}$, the convergence is guaranteed by setting $B_{l}^{(t)} = T^{-a}$. In this case, however, the algorithm converges with a slower rate of $O(\frac{1}{T^{(\frac{1}{2}-a)}})$ with respect to the number of communication rounds. On the other hand, as we discussed in Definition \ref{definitioncompressor}, the sparsity of the compressor is proportional to $B$. As an example, if we use $log_{2}(d)$ bits to encode the positions of the non-zero entries of the compressed gradients, the communication overhead for each communication round is proportional to $B$ as well. Therefore, the convergence rate with respect to the communication overhead remains the same.
\end{Remark}
\section{Details about the Baselines}\label{DetailBaselines}
We compare the proposed methods with the following baselines.
\begin{itemize}
    \item {\scriptsize SIGN}SGD \cite{bernstein2018signsgd1}: each worker adopts the $sign$ compressor and transmits the signs of the gradients.
    \item {Scaled {\scriptsize SIGN}SGD} \cite{karimireddy2019error}: each worker $m$ transmits $\frac{||\boldsymbol{g}_{m}^{(t)}||_{1}}{d}sign(\boldsymbol{g}_{m}^{(t)})$ to the parameter server, in which $\boldsymbol{g}_{m}^{(t)}$ is the local stochastic gradient.
    \item {Noisy {\scriptsize SIGN}SGD} \cite{chen2019distributed}: each worker $m$ first adds a zero-mean Gaussian noise $\boldsymbol{n}$ to the gradients, and then applies the $sign$ compressor. More specifically, it transmits $sign(\boldsymbol{g}_{m}^{(t)}+\boldsymbol{n})$. In our experiments, we tune the variance of the noise from the set $\{0.001,0.01,0.1,1.0\}$ and present the best results.
    \item QSGD \cite{alistarh2017qsgd}: For the gradient $\boldsymbol{g}_{m}^{(t)}$, worker $m$ compresses it to
    \begin{equation}
        Q_{s}(\boldsymbol{g}_{m}^{(t)},s) = ||\boldsymbol{g}_{m}^{(t)}||_{2}sign(\boldsymbol{g}_{m}^{(t)})\xi(\boldsymbol{g}_{m}^{(t)},s),
    \end{equation}
    in which $s$ is the quantization level and $\xi(\boldsymbol{g}_{m}^{(t)},s)$ is defined as follows.
    \begin{equation}
        \xi(\boldsymbol{g}_{m,i}^{(t)},s) =
        \begin{cases}
        \frac{l}{s}, \hfill ~~~\text{w.p. $1-\frac{|\boldsymbol{g}_{m,i}^{(t)}|s}{||\boldsymbol{g}_{m,i}^{(t)}||_2}+l$}, \\
        \frac{l+1}{s}, \hfill ~~~\text{otherwise},
        \end{cases}
    \end{equation}
    in which $0 \leq l < s$ is some integer such that $\frac{|\boldsymbol{g}_{m,i}^{(t)}|}{||\boldsymbol{g}_{m,i}^{(t)}||_2} \in [\frac{l}{s},\frac{l+1}{s}]$.
    \item 1-bit $L_2$ norm QSGD: we set $s=1$ in QSGD.
    \item 1-bit $L_{\infty}$ norm QSGD, we set $s=1$ and replace $||\boldsymbol{g}_{m,i}^{(t)}||_2$ with $||\boldsymbol{g}_{m,i}^{(t)}||_{\infty}$ in QSGD.
    \item TernGrad \cite{wen2017terngrad}: For the gradient $\boldsymbol{g}_{m}^{(t)}$, worker $m$ compresses it to
    \begin{equation}
        ternarize(\boldsymbol{g}_{m}^{(t)}) = s_{t}sign(\boldsymbol{g}_{m}^{(t)})\xi(\boldsymbol{g}_{m}^{(t)},s_{t}),
    \end{equation}
    in which $s_{t} = \max_{m}||\boldsymbol{g}_{m}^{(t)}||_{\infty}$ and $\xi(\boldsymbol{g}_{m}^{(t)},s_{t})$ is defined as follows.
    \begin{equation}
        \xi(\boldsymbol{g}_{m,i}^{(t)},s_{t}) =
        \begin{cases}
        1, \hfill ~~~\text{w.p. $\frac{|\boldsymbol{g}_{m,i}^{(t)}|s}{s_{t}}$}, \\
        0, \hfill ~~~\text{otherwise}.
        \end{cases}
    \end{equation}
\end{itemize}

\section{Details of the Implementation}\label{DetailsImplementation}
Our experiments are mainly implemented using Python 3.8 with packages tensorflow 2.4.1 and numpy 1.19.2. Two Intel(R) Xeon(R) Platinum 8280 CPUs and 8 Tesla V100 GPUs are used in the experiments.
\subsection{Dataset and Pre-processing}
We perform experiments on the standard MNIST dataset, the CIFAR-10 dataset and the CIFAR-100 dataset. MNIST is for handwritten digit recognition consisting of 60,000 training samples and 10,000 testing samples. Each sample is a 28$\times$28 size gray-level image. We normalize the data by dividing it with the max RGB value (i.e., 255.0). The  CIFAR-10 dataset contains 50,000 training samples and 10,000 testing samples. Each sample is a 32$\times$32 color image. The CIFAR-100 dataset is similar to CIFAR-10, but with 100 classes. The data are normalized with zeor-centered mean.
32x32 color images.


\subsection{Neural Network Setting}
For Fashion-MNIST, we implement a three-layer fully connected neural network with softmax of classes with cross-entropy loss. The two hidden layers has 256 and 128 hidden ReLU units, respectively. For CIFAR-10, we implement VGG-9 with 7 convolution layers. It has two contiguous blocks of two convolution layers with 64 and 128 channels, respectively, followed by a max-pooling, then it has one blocks of three convolution layers with 256 channels followed by max-pooling, and at last, we have one dense layer with 512 hidden units. For CIFAR-100, we implement VGG-11 with 8 convolution layers. It has four contiguous blocks of two convolution layers with 64, 128, 256 and 512 channels, respectively, followed by a max-pooling, then it has two dense layers with 1024 hidden units.


\section{Results on CIFAR-100}\label{AdditionalResults}
 In this section, we present the results on the CIFAR-100 dataset. we tune the initial learning rate from the set $\{0.0001, 0.001, 0.01, 0.1, 1.0\}$, which is reduced by factors of 2, 5, and 10 at communication round 1,000, 3,000, and 4,500, respectively. We consider a scenario of $M=100$ normal workers and examine the scenarios of $\alpha \in \{0.1,0.3,0.6,1.0\}$. It can be observed from Table \ref{table_cifar100_local1}-\ref{table_cifar100_local4} that {\scriptsize EF-SPARSIGN}SGD outperforms FedCom in all the examined scenarios, which validates the effectiveness of the proposed method.

\begin{table*}[h]
\caption{Learning Performance on CIFAR-100 ($\alpha = 0.1$)}
\label{table_cifar100_local1}
\begin{center}
\begin{sc}
\begin{tabular}{cccccc}
\toprule
Algorithm & \makecell{Final Accuracy} & \makecell{Communications Rounds \\to Achieve 40\%} & \makecell{Communication Overhead \\to Achieve 40\% (bits)}\\
\midrule
\makecell{FedCom-Local5} & 39.76$\pm$0.27\%& N.A. & N.A.\\
\makecell{FedCom-Local10} & 40.45$\pm$0.16\%& 4,200 & $1.76\times10^{10}$\\
\makecell{FedCom-Local20} & 40.65$\pm$0.67\%& 4,225 & $1.77\times10^{10}$\\
\makecell{{\scriptsize EF-SPARSIGN}SGD-Local5} & 43.13$\pm$0.51\%& 1,350 & $1.71\times10^{9}$\\
\makecell{{\scriptsize EF-SPARSIGN}SGD-Local10} & 46.65$\pm$0.43\%& 1,125 & $1.52\times10^{9}$\\
\makecell{{\scriptsize EF-SPARSIGN}SGD-Local20} & 46.83$\pm$0.43\%& 1,300 & $1.41\times10^{9}$\\
\bottomrule
\end{tabular}
\end{sc}
\end{center}
\end{table*}

\begin{table*}
\caption{Learning Performance on CIFAR-100 ($\alpha = 0.3$)}
\label{table_cifar100_local2}
\begin{center}
\begin{sc}
\begin{tabular}{cccccc}
\toprule
Algorithm & \makecell{Final Accuracy} & \makecell{Communications Rounds \\to Achieve 40\%} & \makecell{Communication Overhead \\to Achieve 40\% (bits)}\\
\midrule
\makecell{FedCom-Local5} & 42.50$\pm$0.63\%& 2,050 & $8.58\times10^{9}$\\
\makecell{FedCom-Local10} & 42.35$\pm$0.29\%& 1,325 & $5.55\times10^{9}$\\
\makecell{FedCom-Local20} & 42.41$\pm$0.57\%& 1,400 & $5.86\times10^{9}$\\
\makecell{{\scriptsize EF-SPARSIGN}SGD-Local5} & 51.66$\pm$0.52\%& 1,025 & $1.38\times10^{9}$\\
\makecell{{\scriptsize EF-SPARSIGN}SGD-Local10} & 52.37$\pm$0.31\%& 825 & $1.12\times10^{9}$\\
\makecell{{\scriptsize EF-SPARSIGN}SGD-Local20} & 52.16$\pm$0.30\%& 925 & $9.71\times10^{8}$\\
\bottomrule
\end{tabular}
\end{sc}
\end{center}
\end{table*}

\begin{table*}
\caption{Learning Performance on CIFAR-100 ($\alpha = 0.6$)}
\label{table_cifar100_local3}
\begin{center}
\begin{sc}
\begin{tabular}{cccccc}
\toprule
Algorithm & \makecell{Final Accuracy} & \makecell{Communications Rounds \\to Achieve 40\%} & \makecell{Communication Overhead \\to Achieve 40\% (bits)}\\
\midrule
\makecell{FedCom-Local5} & 42.57$\pm$0.24\%& 1,775 & $7.43\times10^{9}$\\
\makecell{FedCom-Local10} & 42.71$\pm$0.37\%& 1,050 & $4.40\times10^{9}$\\
\makecell{FedCom-Local20} & 43.62$\pm$0.39\%& 1,025 & $4.29\times10^{9}$\\
\makecell{{\scriptsize EF-SPARSIGN}SGD-Local5} & 51.36$\pm$0.20\%& 1,025 & $1.33\times10^{9}$\\
\makecell{{\scriptsize EF-SPARSIGN}SGD-Local10} & 52.59$\pm$0.06\%& 875 & $1.15\times10^{9}$\\
\makecell{{\scriptsize EF-SPARSIGN}SGD-Local20} & 51.41$\pm$0.17\%& 1,025 & $1.07\times10^{9}$\\
\bottomrule
\end{tabular}
\end{sc}
\end{center}
\end{table*}

\begin{table*}
\caption{Learning Performance on CIFAR-100 ($\alpha = 1.0$)}
\label{table_cifar100_local4}
\begin{center}
\begin{sc}
\begin{tabular}{cccccc}
\toprule
Algorithm & \makecell{Final Accuracy} & \makecell{Communications Rounds \\to Achieve 40\%} & \makecell{Communication Overhead \\to Achieve 40\% (bits)}\\
\midrule
\makecell{FedCom-Local5} & 40.90$\pm$0.97\%& 3,075 & $1.29\times10^{10}$\\
\makecell{FedCom-Local10} & 42.40$\pm$0.31\%& 1,075 & $4.50\times10^{9}$\\
\makecell{FedCom-Local20} & 42.59$\pm$0.65\%& 1,025 & $4.29\times10^{9}$\\
\makecell{{\scriptsize EF-SPARSIGN}SGD-Local5} & 51.01$\pm$0.16\%& 1,025 & $1.33\times10^{9}$\\
\makecell{{\scriptsize EF-SPARSIGN}SGD-Local10} & 52.17$\pm$0.22\%& 875 & $1.10\times10^{9}$\\
\makecell{{\scriptsize EF-SPARSIGN}SGD-Local20} & 51.37$\pm$0.29\%& 1,025 & $1.01\times10^{9}$\\
\bottomrule
\end{tabular}
\end{sc}
\end{center}
\end{table*}

\end{document}